\documentclass{ecai}

\pdfoutput=1 

\usepackage{article}

\begin{document}

\begin{frontmatter}

\title{A Full DAG Score-Based Algorithm for Learning Causal Bayesian Networks with
  Latent Confounders}

\author{\fnms{Christophe}~\snm{Gonzales}}
\author{\fnms{Amir-Hosein}~\snm{Valizadeh}}
\address{Aix Marseille Univ, CNRS, LIS, Marseille, France\\
christophe.gonzales@lis-lab.fr, amir.valizadeh@lis-lab.fr}

\begin{abstract}
  Causal Bayesian networks (CBN) are popular graphical probabilistic models that
  encode causal relations among variables. Learning their graphical
  structure from observational data has received a lot of attention in the
  literature. When there exists no latent (unobserved) confounder, i.e., no
  unobserved direct common cause of some observed variables, learning
  algorithms can be divided essentially into two classes: constraint-based
  and score-based approaches. The latter are often thought to be more
  robust than the former and to produce better results. However, to the best
  of our knowledge, when variables are discrete, no score-based
  algorithm is capable of dealing with latent confounders. This paper
  introduces the first fully score-based structure learning algorithm
  searching the space of DAGs (directed acyclic graphs) that is capable 
  of identifying the presence of some latent confounders. It is justified
  mathematically and  experiments highlight its effectiveness.
\end{abstract}

\end{frontmatter}

\section{Introduction}\label{intro_sec}

Causal networks, a.k.a.\ causal Bayesian networks (CBN) \cite{pear09}, are
graphical probabilistic models that encode cause-and-effect
relationships. Like Bayesian networks (BN), they
are constituted by i)~a directed acyclic graph (DAG) whose nodes represent random
variables and whose edges encode their relationships; and
ii)~a set of conditional probability distributions of the nodes/random
variables given their parents in the graph. However, unlike BNs, the
semantics of the edges is not merely correlation but rather a causal relationship,
that is, an arc from $A$ to $B$ states that $A$ is a direct cause of $B$.
CBNs are important for Artificial Intelligence because they
enable to perform the same kind of reasoning as humans do, in particular
counterfactual reasoning (if I had done this, what would have
happened?).

Although it is well-known that learning the structure of CBNs
from only (observational) data is theoretically not always possible
\cite{pear09}, many algorithms have been proposed in the literature for
this purpose. When there exists no unmeasured {\em confounder}, i.e., no
unobserved direct common cause of some measured variables, they can be
essentially divided into two classes: 
constraint-based and score-based approaches. The former
\cite{spir-glym91,rams-spir-zhan06,colo-maat14,vern-sell-affe-sing-isam17}
rely on statistical conditional independence tests to uncover the
independence properties underlying the probability distribution that generated 
the data, thereby learning the graphical structure of the causal
model. These methods are often not able to uncover the whole
DAG of the CBN, so they provide weaker information in the form of a
Completed Partially Directed Acyclic Graph (CPDAG). In such a graph, only 
the directed edges represent ``true'' causal relations, the undirected ones
representing correlations, their causal direction remaining unknown.
On the other hand, score-based approaches
\cite{chic02,coop-hers92,beek-hoff15} identify 
the DAG of the CBN as the one maximizing some fitness criterion on the
data. They rely on either approximate or exact optimization techniques
to uncover the searched DAG. However the orientations of its arcs
may not always have a causal meaning. So this DAG is mapped into the CPDAG
of its Markov equivalence class, which is the best that can be extracted in
terms of causality from the data. Score-based approaches are usually
considered more robust than constraint-based approaches, notably because, in
the latter, errors in statistical tests can chain and decrease significantly
the quality of the resulting CPDAGs.

However, in most practical situations, some variables play
an important role in the causal mechanism and, yet, for different reasons,
they are not
or cannot be observed in the data. For instance, their measuring may be too
expensive or it would require unethical processes. Constraint-based methods have been
successfully extended to cope with such latent (unobserved) variables
\cite{colo-maat-kali-rich12,spir-glym-schei00,zhan08}. For
score-based algorithms, it is somewhat different: it is commonly admitted that
they are unable to cope with latent variables
because they rely on searching for DAGs and DAGs are inadequate in
the presence of latent variables. An extension of DAGs called
Maximal Ancestral Graphs (MAG) \cite{rich-spir02} has been introduced
precisely to fix this issue and
score-based approaches have been adapted to learn MAGs
\cite{rich-spir98,ogar-spir-rams16,trian-tsam16}. Unfortunately, currently,
they can only cope with scoring MAGs over continuous variables. Yet, this
is restrictive because there exist situations in which variables are
discrete by nature and cannot be meaningfully extended as continuous ones, e.g.,
non-ordinal variables such as colors, locations, types of
devices, {\em etc.}

In this paper, we address problems in which all the random variables are
discrete. In \cite{tian-pear02},it was shown that causal models with
arbitrary latent variables can always be converted into semi-Markovian
causal models (SMCM), i.e., models in which latent variables have no parent and
only two children, while preserving the same independence relations between
the observed variables. So, to deal with latent confounders, we
focus on learning SMCMs. More precisely, we show that, without any prior
knowledge about the latent confounders or their number, DAGs learnt from
observational data by latent confounders-unaware
score-based approaches encode sufficient information to recover
many latent confounders and their locations. Exploiting this property,
we provide and justify a structure learning algorithm that i)~only relies on scores;
ii)~uses only DAGs; and iii)~is capable of identifying some latent
confounders and their locations. 

The rest of the paper is organized as follows. Section~\ref{related_sec}
presents formally causal models and some algorithms for learning BN and/or
CBN structures from observational data that can cope with latent
confounders. Then, in Section~\ref{algo_sec}, we 
introduce our new algorithm and justify why it is capable of identifying
latent confounders. Its effectiveness is highlighted through experiments in
Section~\ref{expe_sec}. Finally a conclusion and some future
works are provided in Section~\ref{conclu_sec}.


\section{Causal Models and Structure Learning}\label{related_sec}

In the paper, bold letters represent sets. $\XXX$ denotes a set
{\em discrete} random variables. For a directed graph $\GGG$, 
$\Ch_{\GGG}(X)$ and $\Pa_{\GGG}(X)$ denote the set of children and parents of node
$X$ respectively, i.e., the set of nodes $Y$ such that there exists an arc
from $X$ to $Y$ and from $Y$ to $X$ respectively.  A causal model over $\XXX$ is
defined as follows \cite{pear09}: 

\begin{definition}\label{CBN_def}
  A causal model is a pair $(G,\mathbf{\Theta})$ where $G=(\XXX,\EEE)$ is a
  DAG\footnote{By abuse of notation, we use
    interchangeably $X \in \XXX$ to denote a node in the model and its
    corresponding random variable.} and $\EEE$ is a set of arcs. To each
  $X_i \in \XXX$ is assigned a random disturbance $\xi_i$.
  $\mathbf{\Theta} = \{f_i(\Pa_{\GGG}(X_i),\xi_i)\}_{X_i \in \XXX} \cup
  \{P(\xi_i) \}_{X_i \in \XXX}$, where $f_i$'s are
  functions assigned to $X_i$'s and $P$ are
  probability distributions over disturbances $\xi_i$.
\end{definition}

It is easy to see\footnote{See the supplementary material, Appendix~C, at
  the end of the paper.}
that a causal model can be represented equivalently by a Bayesian network -- BN
\cite{pear88}:

\begin{definition}\label{BN_def}
  A BN is a pair $(G,\mathbf{\Theta})$ where $G=(\XXX,\EEE)$ is a
  directed acyclic graph (DAG), $\XXX$ represents a set of random
  variables, $\EEE$ is a set of arcs, and $\mathbf{\Theta} =
  \{P(X|\Pa_{\GGG}(X))\}_{X \in \mathbf{\XXX}}$ 
  is the set of the conditional probability distributions (CPD) of the
  nodes / random variables $X$ in $\GGG$ given their parents $\Pa_{\GGG}(X)$ in
  $\GGG$. The BN encodes the joint probability over $\XXX$ as
  $P(\XXX) = \prod_{X \in \XXX} P(X|\Pa_{\GGG}(X))$.
\end{definition}

Causal models impose that the arcs are oriented in
the direction of causality, that is, an arc $X \rightarrow Y$ means that
$X$ is a direct cause of $Y$. In this case, the BN is called a CBN. In
general, BNs do not impose this restriction since
they only model probabilistic dependences. Hence a BN containing only Arc $X
\rightarrow Y$ is equivalent to one containing Arc $Y \rightarrow X$. More
precisely, the independence model of a BN is specified by the
$d$-separation criterion
\cite{pear88}:

\begin{definition}[Trails and $d$-separation]\label{dsep_def}
  Let $\GGG$ be a DAG. A trail $\CCC$ between nodes $X$
  and $Y$ is a sequence of nodes $\langle X_1=X,\ldots,X_k=Y \rangle$ such 
  that, for every $i \in \{1,\ldots,k-1\}$, $\GGG$ contains either Arc
  $X_i \rightarrow X_{i+1}$ or Arc $X_i \leftarrow X_{i+1}$.

  Let $\mathbf{Z}$ be a set of nodes disjoint from $\{X,Y\}$. $X$ and $Y$
  are said to be $d$-separated by $\mathbf{Z}$, which is denoted by
  $\condindepd{X}{Y}{\mathbf{Z}}{\GGG}$, if, for every trail $\CCC$
  between $X$ and $Y$, there exists a node $X_i\in \CCC$, $i \not\in \{1,k\}$,
  such that one of the 
  following two conditions holds: 
  \begin{enumerate}
  \item $\langle X_{i-1}, X_i, X_{i+1} \rangle$ is a collider,
    i.e., $\GGG$ contains Arcs $X_{i-1} \rightarrow X_{i}$ and $X_i
    \leftarrow X_{i+1}$. In addition, neither $X_i$ nor its descendants in $\GGG$
    belong to $\mathbf{Z}$. The descendants of a node are defined recursively
    as the union of its children and the descendants of these children.
  \item $\langle X_{i-1}, X_i, X_{i+1} \rangle$ is not a
    collider and $X_i$ belongs to $\mathbf{Z}$. 
  \end{enumerate}
  Such trails are called {\em blocked}, else they are {\em active}.
  Let $\mathbf{U}, \mathbf{V}, \mathbf{Z}$ be disjoint sets of nodes. Then
  $\mathbf{U}$ and $\mathbf{V}$ are $d$-separated by
  $\mathbf{Z}$ if and only if $X$ and $Y$ are $d$-separated by $\mathbf{Z}$
  for all $X \in \mathbf{U}$ and all $Y \in \mathbf{V}$.
\end{definition}

Let $\mathbf{U}, \mathbf{V}, \mathbf{Z} \subseteq \XXX$ be disjoint sets
and let $P$ be a probability distribution over $\XXX$. 
We denote by $\mathbf{U} \indep_P \mathbf{V} | \mathbf{Z}$ the
probabilistic conditional
independence of $\mathbf{U}$ and $\mathbf{V}$ given $\mathbf{Z}$. The
independence model of BNs is the following:
\begin{equation}\label{indep_eq}
  \condindepd{\mathbf{U}}{\mathbf{V}}{\mathbf{Z}}{\GGG} \Longrightarrow
  \mathbf{U} \indep_P \mathbf{V} | \mathbf{Z}.
\end{equation}
Two BNs with the same set of $d$-separation
properties therefore represent the same independence model.
Any graphical model satisfying Eq.~(\ref{indep_eq}) is called an I-map
({\em independence map}). The $d$-separation criterion implies
that two BNs represent the same
independence model if and only if they have the same {\em skeleton} and the
same set of {\em v-structures} \cite{verm-pear90}: the skeleton of a
directed graph $\GGG$ is the undirected graph obtained by removing all
the orientations from the arcs of $\GGG$; v-structures are 
colliders $\langle X_{i-1}, X_i, X_{i+1} \rangle$ such that $\GGG$ does not
contain any arc between $X_{i-1}$ and $X_{i+1}$. The directions of
the arcs in v-structures are therefore identical for all BNs that represent
the same sets of independences. When Eq.~(\ref{indep_eq}) is substituted
by:
\begin{equation}\label{pmap_eq}
  \condindepd{\mathbf{U}}{\mathbf{V}}{\mathbf{Z}}{\GGG} \Longleftrightarrow
  \mathbf{U} \indep_P \mathbf{V} | \mathbf{Z},
\end{equation}
then the graphical model is called a P-map ({\em perfect map}).

To learn the structure of a BN from observational data, it is usually
assumed that the distribution $P$ that generated the data has a P-map
(although there exist some papers that relax this assumption
\cite{rams-spir-zhan06,leme-mega-cart-liu12,mabr-gonz-jabe-choj14}).
Then, it is sufficient to learn the independence model of the data (the set
of conditional independences): each independence $X \indep_P Y |
\mathbf{Z}$ necessarily implies the lack of an arc between $X$ and
$Y$ in the BN. The (undirected) edges of the skeleton therefore correspond
to all the pairs of nodes $(X,Y)$ for which no conditional independence was
found. The set of v-structures $\langle X, Z, Y \rangle$ is the set of triples
$(X,Z,Y)$ such that i)~edges $X-Z$ and $Z-Y$ belong to the skeleton and
ii)~there exist sets $\mathbf{Z} \not\supseteq \{Z\}$ such that  
$X \indep_P Y | \mathbf{Z}$. The edges of the skeleton corresponding to
v-structures can be oriented accordingly. Now, to avoid creating
additional spurious v-structures or directed cycles (which are forbidden in
DAGs), some edges need necessarily be oriented in a given direction. These
are identified using Meek's rules \cite{meek95} and can be computed in
polynomial time \cite{chic95}. The graph resulting from
all these orientations is called a CPDAG ({\em Completed Partially Directed
  Acyclic Graph}). To complete the learning of the BN's structure, there
just remains to orient the remaining undirected edges. To do so, it is sufficient to
sequentially apply the following two operations until there
remains no undirected edge: i)~orient in any direction one (arbitrary) edge; and
ii)~apply Meek's rules. This is precisely what {\em constraint-based}
algorithms like PC \cite{spir-glym91}, PC-stable
\cite{colo-maat14}, CPC \cite{rams-spir-zhan06}, IC \cite{verm-pear90} or FCI
\cite{spir-glym-schei00} do, relying on statistical independence
tests. MIIC \cite{vern-sell-affe-sing-isam17} essentially performs the same
operations but exploiting multivariate information instead.

There exist many other algorithms, based notably on learning
Markov blankets \cite{tsam-alif-stat03} or on scoring DAGs
\cite{coop-hers92,heck-geig-chic95,chic02,beek-hoff15}. The key idea of
score-based approaches is to assign to each DAG a score representing its
fitness on the data. Under some assumptions, the one maximizing this
criterion is the one that generated the data. In the rest of the paper, we
exploit the BIC score \cite{schw78}:
\begin{displaymath}
  S(X|\mathbf{Z}) = \sum_{x \in \dom{X}} \sum_{\mathbf{z} \in \dom{\mathbf{Z}}}
  N_{x\mathbf{z}} \log\left(\frac{N_{x\mathbf{z}}}{N_{\mathbf{z}}}\right) -
  \frac{1}{2} \log(|\DDD|) dim(X|\mathbf{Z}),
\end{displaymath}
where $S(X|\mathbf{Z})$ denotes the score of node $X$ given parent set
$\mathbf{Z}$; $\dom{X}$ and $\dom{\mathbf{Z}}$ represent the domains of
$X$ and $\mathbf{Z}$ repectively; $N_{x\mathbf{z}}$ is the number of
records in Database $\DDD$ such that $X=x$ and $\mathbf{Z} = \mathbf{z}$;
$N_{\mathbf{z}} = \sum_{x \in \dom{X}} N_{x\mathbf{z}}$; and, finally,
$dim(X|\mathbf{Z})$ is the number of free parameters in the conditional
probability distribution $P(X|\mathbf{Z})$, i.e., 
$dim(X|\mathbf{Z}) = (|\dom{X}|-1) \times |\dom{\mathbf{Z}}|$. The term
$\frac{1}{2} \log(|\DDD|) dim(X|\mathbf{Z})$ is called the {\em penalty} of the score.

In the literature, v-structures are considered to represent causal
relations, that is, the directions of their arcs have a causal
meaning. This implies that, when there exist no latent confounder, CPDAGs
are precisely the best that can be extracted in terms of causality from the
data. So, all the aforementioned algorithms do actually learn CBNs when
they return CPDAGs instead of full BNs. 

In practice, for different reasons (ethics, price, immeasurability, {\em
  etc.}), it is often the case that some confounders cannot be observed.
Constraint-based approaches, notably FCI, have been extended to deal with such
situations. Instead of returning a CPDAG, they provide more informative 
graphs like Partial Ancestral Graphs (PAG) \cite{rich-spir98,spir-glym-schei00}.
PAGs are constituted by usual arcs ($\rightarrow$) but also labeled edges
(\mbox{$\circ\!\!-\!\!\circ$}, \mbox{$\circ\!\!\rightarrow$}) and bidirected arcs
($\leftrightarrow$). The latter indicate the presence of a confounder, that
is, $A \leftrightarrow B$ in a PAG means that there exists $A \leftarrow L
\rightarrow B$ in the generating DAG, with $L$ the confounder. The $\circ$
labels indicate that the learning algorithm is uncertain whether there
should be an arrow head or not, i.e., $\circ\!-$ is equivalent to either
$\leftarrow$ or $-$. For
continuous random variables, score-based approaches have been extended to
cope with latent confounders \cite{rich-spir98}. But, to our knowledge,
when variables are discrete, there exists no score-based approach
capable of dealing with latent confounders. One reason is that it is
believed that detecting confounders is 
impossible when searching with scores the space of DAGs. We show in the next
section that it is not the case.


\section{A New Full Score-based Causal Learning Algorithm}\label{algo_sec}

In the rest of the paper, $\XXX$ is
divided into $\XXX_O$ and $\XXX_H$, which represent sets of observed and
hidden (latent) variables respectively. We assume that there exists
an underlying probability distribution $P^*$ over $\XXX = \XXX_O \cup \XXX_H$ 
that generated a dataset $\DDD^*$. But, as the variables in
$\XXX_H$ are unobserved (they are the latent confounders), only the projection
$\DDD$ of $\DDD^*$ over $\XXX_O$, i.e., the dataset resulting from the
removal from $\DDD^*$ of all the values of the variables of $\XXX_H$, is
available for learning. In addition, we assume that $P^*$ is decomposable
according to some DAG $\GGG^*$. The goal is to recover $\GGG^*$.

The key idea of our algorithm is summarized on Figure~\ref{triangle_fig}:
the left side displays part of Graph $\GGG^*$, which contains a latent confounder
$L \in \XXX_H$ and represents the structure of the causal network that generated the data;
Graph $\GGG$ on the right should be the one learnt by a score-based algorithm.
Indeed, in $\GGG^*$, there exists no set $\mathbf{Z} \subseteq \XXX_O$ that
$d$-separates $A$ and $B$ (because Trail $\langle A,L,B \rangle$ is active).
Hence, provided $\GGG^*$ is a P-map, $A$ and $B$  should be dependent and a structure 
with an arc between $A$ and $B$ should have a higher score than one without.
Assume that this arc is $A \rightarrow B$. For the same reason, a structure
with an arc between $A$ and $C$ (resp. $B$ and $D$) should have a higher
score that one without. Now, given any $\mathbf{Z} \supseteq \{A\}$, node
$B$ is not $d$-separated from 
$C$ in $\GGG^*$ because Trail $\langle C,A,L,B \rangle$ is active.
But it would be on Fig.~\ref{triangle_fig}.b if there
were no arc between $B$ and $C$. This is the reason why score-based
algorithms tend to produce structures with such an arc and analyzing such
triangles $(A,B,C)$ in the learnt graph 
should provide some insight on the location of the latent confounders. This
intuition is confirmed by the next proposition.

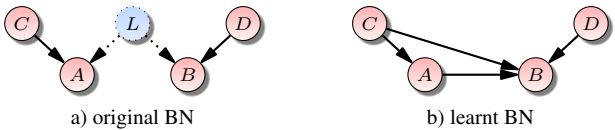
\begin{figure}[htb]
  \centerline{
    \begin{tikzpicture}[
      mynode/.style={circle, draw, font=\scriptsize,
        bottom color = red!5, top color = red!30,
        black, circular drop shadow , text = black, 
        inner sep=0.4mm, minimum size=4.5mm,
        node distance=3.5mm and 4mm},      
      mynode2/.style={circle, dotted, draw, font=\scriptsize,
        bottom color = BlockBodyColor!50, top color = BlockBodyColor,
        black, circular drop shadow , text = black, 
        inner sep=0.4mm, minimum size=4.5mm,
        node distance=3.5mm and 4mm},      
      mynode3/.style={circle, draw, font=\scriptsize,
        bottom color = white, top color = white,
        white, text = white, 
        inner sep=0.4mm, minimum size=4.5mm,
        node distance=3.5mm and 4mm},
      myedge/.style={thick, ->,
        >={Stealth[inset=0pt,length=8pt,angle'=28,round]}},
      myedge2/.style={myedge, dotted}
      ]
      \node[mynode] (C1) {$C$};
      \node[mynode, below right= of C1] (A1) {$A$};
      \node[mynode2, above right= of A1] (L1) {$L$};
      \node[mynode, below right= of L1] (B1) {$B$};
      \node[mynode, above right= of B1] (D1) {$D$};
      \draw[myedge] (C1) -- (A1);
      \draw[myedge] (D1) -- (B1);
      \draw[myedge2] (L1) -- (A1);
      \draw[myedge2] (L1) -- (B1);
      \node[below=8mm of L1] {\fontsize {8}{8pt}\selectfont a) original BN};
      
      \node[mynode, right=12mm of D1] (C2) {$C$};
      \node[mynode, below right= of C2] (A2) {$A$};
      \node[mynode3, above right= of A2] (L2) {$L$};
      \node[mynode, below right= of L2] (B2) {$B$};
      \node[mynode, above right= of B2] (D2) {$D$};
      \draw[myedge] (C2) -- (A2);
      \draw[myedge] (C2) -- (B2);
      \draw[myedge] (A2) -- (B2);
      \draw[myedge] (D2) -- (B2);
      \node[below=8mm of L2] {\fontsize {8}{8pt}\selectfont b) learnt BN};
    \end{tikzpicture} 
  }
  \caption{\label{triangle_fig}Triangles induced by latent confounders.}
\end{figure}

\begin{proposition}\label{prop_triangle}
  Assume that there exists a perfect map $\GGG^* = (\XXX,\EEE)$ for
  Distribution $P^*$ and that $\XXX_H$ is the set of latent confounders,
  i.e., all the nodes of $\XXX_H$ have no parent and exactly two children
  in $\GGG^*$, which both belong to $\XXX_O$.

  Let $L \in \XXX_H$ be any variable such that both of its children
  $A,B$ in $\GGG^*$ have at least one parent in $\XXX_O$. Then, if $\GGG$
  is a DAG maximizing the BIC score over $\DDD$, as $|\DDD| \rightarrow
  \infty$, $\GGG$ contains an arc between $A$ and $B$. Without loss of
  generality, assume this is Arc $A \rightarrow B$. Then,
  for every $C \in \Pa_{\GGG^*}(A) \cap \XXX_O$, $(A,B,C)$ is a clique in
  $\GGG$.
\end{proposition}

\begin{xproof}
All the proofs are provided in the supplementary material, Appendix~A, at
the end of the paper.
\end{xproof}

Note that, in Proposition~\ref{prop_triangle}, both $A$ and $B$ have other parents
than $L$. This is important because, as shown in Proposition~\ref{prop_indistinguishable1}
below, whenever node $L$ is the only parent of $A$ (resp.\ $B$), no learning algorithm can
distinguish between a graph $\GGG^*$ with a latent confounder $L$ whose
children are $A$ and $B$ (Fig.~\ref{indistinguish_fig}.a \&
~\ref{indistinguish_fig}.b), and a graph $\GGG^*$ without latent confounder $L$
but with an arc $A \rightarrow B$ (resp.\ $B \rightarrow A$) (Fig.~\ref{indistinguish_fig}.c).
Similarly,  as shown in Proposition~\ref{prop_indistinguishable2}, it is
impossible to distinguish between a graph $\GGG^*$ containing both latent confounder
$L$ and Arc $A \rightarrow B$ (Fig.~\ref{indistinguish_fig}.d), and a graph
$\GGG^*$ without latent confounder $L$ but with both arcs $A \rightarrow B$ and
$C \rightarrow B$ (Fig.~\ref{indistinguish_fig}.e). 

\begin{figure}[htb]
  \centerline{
    \begin{tikzpicture}[
      mynode/.style={circle, draw, font=\scriptsize,
        bottom color = red!5, top color = red!30,
        black, circular drop shadow , text = black, 
        inner sep=0.4mm, minimum size=4.5mm,
        node distance=4mm and 3mm},      
      mynode2/.style={circle, dotted, draw, font=\scriptsize,
        bottom color = BlockBodyColor!50, top color = BlockBodyColor,
        black, circular drop shadow , text = black, 
        inner sep=0.4mm, minimum size=4.5mm,
        node distance=3.5mm and 4mm},      
      mynode3/.style={circle, draw, font=\scriptsize,
        bottom color = white, top color = white,
        white, text = white, 
        inner sep=0.4mm, minimum size=4.5mm,
        node distance=3.5mm and 4mm},
      myedge/.style={thick, ->,
        >={Stealth[inset=0pt,length=8pt,angle'=28,round]}},
      myedge2/.style={myedge, dotted}
      ]
      \node[mynode2] (L1) {$L$};
      \node[mynode, below left=  of L1] (A1) {$A$};
      \node[mynode, below right= of L1] (B1) {$B$};
      \node[mynode, above right= of B1] (D1) {$D$};
      \draw[myedge] (D1) -- (B1);
      \draw[myedge2] (L1) -- (A1);
      \draw[myedge2] (L1) -- (B1);
      \node[below=8mm of L1] {a) structure 1};

      \node[mynode2,right= 22mm of L1] (L2) {$L$};
      \node[mynode, below left=  of L2] (A2) {$A$};
      \node[mynode, below right= of L2] (B2) {$B$};
      \node[mynode, above right= of B2] (D2) {$D$};
      \draw[myedge] (D2) -- (B2);
      \draw[myedge2] (L2) -- (A2);
      \draw[myedge2] (L2) -- (B2);
       \draw[myedge] (A2) -- (B2);
      
      \node[below=8mm of L2] {b) structure 2};

      \node[mynode3,right= 22mm of L2] (L3) {$L$};
      \node[mynode, below left=  of L3] (A3) {$A$};
      \node[mynode, below right= of L3] (B3) {$B$};
      \node[mynode, above right= of B3] (D3) {$D$};
      \draw[myedge] (A3) -- (B3);
      \draw[myedge] (D3) -- (B3);
      \node[below=8mm of L3] {c) structure 3};
    \end{tikzpicture}}

  \centerline{
    \begin{tikzpicture}[
      mynode/.style={circle, draw, font=\scriptsize,
        bottom color = red!5, top color = red!30,
        black, circular drop shadow , text = black, 
        inner sep=0.4mm, minimum size=4.5mm,
        node distance=4mm and 3mm},      
      mynode2/.style={circle, dotted, draw, font=\scriptsize,
        bottom color = BlockBodyColor!50, top color = BlockBodyColor,
        black, circular drop shadow , text = black, 
        inner sep=0.4mm, minimum size=4.5mm,
        node distance=3.5mm and 4mm},      
      mynode3/.style={circle, draw, font=\scriptsize,
        bottom color = white, top color = white,
        white, text = white, 
        inner sep=0.4mm, minimum size=4.5mm,
        node distance=3.5mm and 4mm},
      myedge/.style={thick, ->,
        >={Stealth[inset=0pt,length=8pt,angle'=28,round]}},
      myedge2/.style={myedge, dotted}
      ]
      \node[mynode2] (L1) {$L$};
      \node[mynode, below left=  of L1] (A1) {$A$};
      \node[mynode, above left=  of A1] (C1) {$C$};
      \node[mynode, below right= of L1] (B1) {$B$};
      \node[mynode, above right= of B1] (D1) {$D$};
      \draw[myedge] (A1) -- (B1);
      \draw[myedge] (C1) -- (A1);
      \draw[myedge] (D1) -- (B1);
      \draw[myedge2] (L1) -- (A1);
      \draw[myedge2] (L1) -- (B1);
      \node[below=8mm of L1] {d) structure 4};

      \node[mynode3,right= 29mm of L1] (L2) {$L$};
      \node[mynode, below left=  of L2] (A2) {$A$};
      \node[mynode, above left=  of A2] (C2) {$C$};
      \node[mynode, below right= of L2] (B2) {$B$};
      \node[mynode, above right= of B2] (D2) {$D$};
      \draw[myedge] (A2) -- (B2);
      \draw[myedge] (C2) -- (A2);
      \draw[myedge] (C2) -- (B2);
      \draw[myedge] (D2) -- (B2);
      \node[below=8mm of L2] {e) structure 5};
    \end{tikzpicture}
  }
  \caption{\label{indistinguish_fig}Some indistingushable structures.}
\end{figure}
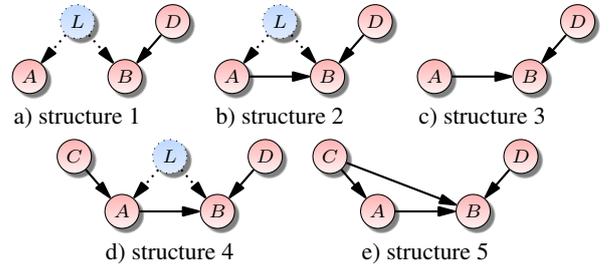

\begin{proposition}\label{prop_indistinguishable1}
  Let $\BBB^* = (\GGG^*,\Btheta)$ be a
  Bayesian network, with $\GGG^* = (\XXX,\EEE)$. Let $L \in \XXX_H$ be such that
  $\Ch_{\GGG^*}(L) = \{A, B\} \subseteq \XXX_O$ and $\Pa_{\GGG^*}(L) = \emptyset$. In
  addition, assume that $\Pa_{\GGG^*}(A) = \{L\}$. Finally, let $\GGG$ be the graph
  resulting from the removal of $L$ (and its outgoing arcs) from $\GGG^*$ and
  the addition of arc $A \rightarrow B$ (if $\GGG^*$ does not already contain
  it). Then, for any triple of disjoint subsets of variables
  $\mathbf{U},\mathbf{V},\mathbf{W}$ of $\XXX_O$, we have that:
  \begin{equation}
    \condindepd{\mathbf{U}}{\mathbf{V}}{\mathbf{W}}{\GGG}
    \Longleftrightarrow
    \condindepd{\mathbf{U}}{\mathbf{V}}{\mathbf{W}}{\GGG^*}.
  \end{equation}
\end{proposition}

\begin{proposition}\label{prop_indistinguishable2}
  Let $\BBB^* = (\GGG^*,\Btheta)$ be a Bayesian network such that $\GGG^* =
  (\XXX,\EEE)$ and Arc $A \rightarrow B$ belongs to $\EEE$. Let $L \in
  \XXX_H$ be such that $\Ch_{\GGG^*}(L) = \{A, B\} \subseteq \XXX_O$. 
  Finally, let $\GGG$ be the graph resulting from the removal of $L$ (and
  its outgoing arcs) from $\GGG^*$ and the addition the set of arcs $\{X
  \rightarrow B : X \in \Pa_{\GGG^*}(A) \backslash \{L\}\}$ (if $\GGG^*$
  does not already contain them). Then, for any triple of disjoint subsets
  of variables $\mathbf{U},\mathbf{V},\mathbf{W}$ of $\XXX_O$, we have that:
  \begin{equation}
    \condindepd{\mathbf{U}}{\mathbf{V}}{\mathbf{W}}{\GGG}
    \Longleftrightarrow
    \condindepd{\mathbf{U}}{\mathbf{V}}{\mathbf{W}}{\GGG^*}.
  \end{equation}
\end{proposition}

In the two situations mentioned above, algorithms like FCI deal with
indistinguishability by not choosing a single structure to return but
rather by labelling arcs with $\circ$ to highlight the uncertainty about
the locations of the arrow heads and by asking the user to personally select
which structure seems the most appropriate. In a sense, this corresponds to
completing the structure learning with some user's expert knowledge.
Other algorithms like MIIC or our algorithm prefer to choose
which structure seems the best, hence relieving the user of such a burden.
The rule followed by our algorithm is to discard latent variables when
they are not absolutely necessary (i.e., in Fig.\ref{indistinguish_fig},
it selects only Graphs \ref{indistinguish_fig}.c \&
\ref{indistinguish_fig}.e). This criterion can be viewed as a simple 
Occam razor. So, without loss of generality, in the rest of the paper, we assume that
i)~there exists no arc between $A$ and $B$ in $\GGG^*$; and ii)~both $A$
and $B$ have other parents than $L$ in $\GGG^*$.

The rationale of our algorithm relies on first learning a
structure using any score-based algorithm and, second, examining the
triangles to discover the latent confounders. The learnt DAG is then updated
to take into account these confounders, mapping Fig.~\ref{triangle_fig}.b
into Fig.~\ref{triangle_fig}.a. Unfortunately,
the original network $\GGG^*$ may itself contain some triangles,
which we will call {\em genuine} triangles. So, we need to discriminate
them from the ones induced by the latent confounders, which we call
{\em latent} triangles. For this purpose, remark that, for every pair
$(X,Y)$ of variables of the genuine triangles, there exists no set
$\mathbf{Z} \subseteq \XXX_O \backslash \{X,Y\}$ such that 
$\condindepd{X}{Y}{\mathbf{Z}}{\GGG^*}$ because, $\GGG^*$ containing an arc
between $X$ and $Y$, Trail $\langle X,Y \rangle$ is active.
So, if $\GGG^*$ is a perfect map and Dataset $\DDD$ is sufficiently large, we should
not find any $\mathbf{Z} \subseteq \XXX_O \backslash \{X,Y\}$ such that
$X \indep_{P^*} Y | \mathbf{Z}$. Here, note that Dataset $\DDD$ being the projection of
$\DDD^*$ on $\XXX_O$, it is generated by a distribution $P$ over $\XXX_O$ defined
as the marginal of $P^*$ over $\XXX_O$, i.e., $P = \sum_{\XXX_H} P^*$.
But since $\XXX_O$ contains $\{X\}$, $\{Y\}$ and $\mathbf{Z}$, joint probability
$P(X,Y,\mathbf{Z}) = P^*(X,Y,\mathbf{Z})$, so that
$X \indep_{P} Y | \mathbf{Z}$ is equivalent to $X \indep_{P^*} Y | \mathbf{Z}$.
Overall, in Dataset $\DDD$, for genuine triangles, it should not be
possible to find any pair of variables $(X,Y)$ in the triangle such that there exists a set
$\mathbf{Z} \subseteq \XXX_O \backslash \{X,Y\}$ such that $X \indep_{P} Y | \mathbf{Z}$.

In latent triangles, for the same reason, this property holds for
pair $(A,C)$. For Pair $(A,B)$, it also holds because Trail $\langle A,L,B \rangle$
is active for all $\mathbf{Z} \subseteq \XXX_O \backslash \{A,B\}$
(see Fig.~\ref{triangle_fig}.a). 
On the contrary, for pair $(B,C)$ of
Fig.~\ref{triangle_fig}.b, there exist $d$-separating sets $\mathbf{Z}
\subseteq \XXX_O \backslash \{B,C\}$ in $\GGG^*$ because $\GGG^*$ does not
contain any arc between $B$ and $C$. Note that $A \not\in \mathbf{Z}$
otherwise Trail $\langle B,L,A,C \rangle$ of Fig.~\ref{triangle_fig}.a would be active and, so,
$\ncondindepd{B}{C}{\mathbf{Z}}{\GGG^*}$. So, in terms of
independences observable in Dataset $\DDD$, there exist $\mathbf{Z}
\subseteq \XXX_O \backslash \{A,B,C\}$ such $B \indep_{P} C | \mathbf{Z}$
and, moreover, $B \notindep_{P} C | \mathbf{Z} \cup \{A\}$. This suggests
the following property to discriminate between latent and genuine
triangles:

\begin{myrule}\label{rule1}
  Only latent triangles contain {\em exactly one pair} of nodes that
  are independent given some set $\mathbf{Z} \subseteq \XXX_O$ but are dependent given $\mathbf{Z}$
  union the third node of the triangle.
\end{myrule}

For the case of Fig.~\ref{triangle_fig}.a, any of the latent
triangles of Fig.~\ref{triangle2_fig} can be learnt by score-based
approaches because they encode exactly the same $d$-separation properties.
In Fig.~\ref{triangle2_fig}, the independent pair of nodes mentioned in
Rule~\ref{rule1} is $(B,C)$ for the three types. 
Triangles of types 2 and 3 share the fact that $\GGG^*$'s arc $A \rightarrow
C$ has been reversed. If $C$ had parents in $\GGG^*$, then these have
become its children in $\GGG$ to avoid creating new v-structures. So, when
triangles of types 2 or 3 are learnt, $C$ should not have many parents
other than $A$, which may not be the case for $B$. So, to minimize the
BIC score's penalty induced by the arc between $B$ and $C$, the latter will
most probably be oriented as $B \rightarrow C$. This intuition was
confirmed in experiments we conducted: Type 2 triangles only allowed us to
determine less than 0.5\% of the latent confounders\footnote{Note that,
  from an observational point of view, it is possible to discriminate
  between Type~2 and Type~3 triangles: in the former, $C$ is a parent of
  $B$ (as in the original graph $\GGG^*$). So, for parents $D$ of $B$ in
  $\GGG^*$, $\langle C, B, D \rangle$ is a v-structure in $\GGG$, which
  fits the independences observable in the dataset. On the contrary, in
  Type 3 triangles, the connection $\langle C, B, D \rangle$ is not a
  v-structure, which contradicts the observable independences. So, as shown
  in Figure~\ref{triangle3_fig}, the learning algorithm should add an arc
  between $D$ and $C$.}. 
So, to speed-up
our algorithm without decreasing significantly its effectiveness, 
we chose to exploit only Types 1 and 3 latent triangles. Note that this
makes our algorithm only approximate but, as shown in the experiments, it
still remains very effective. Rule~\ref{rule2}
summarizes the features of the above triangles:

\vspace{-0.5mm}

\begin{myrule}\label{rule2}
  Triangles of Type~1 and Type~3 are such that
  \raisebox{-1.1mm}{%
    \begin{tikzpicture}
      \node(C) {$C$};
      \node[right=3mm of C] (A) {$A$};
      \node[right=3mm of A] (B) {$B$};
      \draw (C) edge[->] (A);
      \draw (A) edge[->] (B);
      \draw (C) edge[->, bend left=20] (B);
      \useasboundingbox ($(C.south west) + (0,0.5mm)$) rectangle (B.north east);
    \end{tikzpicture}}
  and
  \raisebox{-1.1mm}{%
    \begin{tikzpicture}
      \node(C) {$A$};
      \node[right=3mm of C] (A) {$B$};
      \node[right=3mm of A] (B) {$C$};
      \draw (C) edge[->] (A);
      \draw (A) edge[->] (B);
      \draw (C) edge[->, bend left=20] (B);
      \useasboundingbox ($(C.south west) + (0,0.5mm)$) rectangle (B.north east);
    \end{tikzpicture}}
  respectively.

  For both types, there exists some set
  $\mathbf{Z} \subseteq \XXX_O \backslash \{A,B,C\}$ such that
  $B \indep_{P} C | \mathbf{Z}$ and
  $B \notindep_{P} C | \mathbf{Z} \cup \{A\}$; and there exists no 
  $\mathbf{Z} \subseteq \XXX_O \backslash \{A,B,C\}$ such that
  $A \indep_{P} B | \mathbf{Z}$ or $A \indep_{P} C | \mathbf{Z}$.
\end{myrule}

\begin{figure}[t]
  \centerline{
    \begin{tikzpicture}[
      mynode/.style={circle, draw, font=\scriptsize,
        bottom color = red!5, top color = red!30,
        black, circular drop shadow , text = black, 
        inner sep=0.4mm, minimum size=4.5mm,
        node distance=3.5mm and 2mm},      
      mynode2/.style={circle, dotted, draw, font=\scriptsize,
        bottom color = BlockBodyColor!50, top color = BlockBodyColor,
        black, circular drop shadow , text = black, 
        inner sep=0.4mm, minimum size=4.5mm,
        node distance=3.5mm and 2mm},      
      mynode3/.style={circle, draw, font=\scriptsize,
        bottom color = white, top color = white,
        white, text = white, 
        inner sep=0.4mm, minimum size=4.5mm,
        node distance=3.5mm and 2mm},
      myedge/.style={thick, ->,
        >={Stealth[inset=0pt,length=8pt,angle'=28,round]}},
      myedge2/.style={myedge, dotted}
      ]
      \node[mynode] (C1) {$C$};
      \node[mynode, below=3.5mm of C1, xshift=5mm] (A1) {$A$};
      \node[mynode3, above right= of A1, xshift=1mm] (L1) {$L$};
      \node[mynode, below right= of L1, xshift=2mm] (B1) {$B$};
      \draw[myedge] (C1) -- (A1);
      \draw[myedge] (C1) -- (B1);
      \draw[myedge] (A1) -- (B1);
      \node[below=8mm of L1] {a) Type 1};
   
      \node[mynode, right=21.5mm of C1] (C2) {$C$};
      \node[mynode, below=3.5mm of C2, xshift=5mm] (A2) {$A$};
      \node[mynode3, above right= of A2, xshift=1mm] (L2) {$L$};
      \node[mynode, below right= of L2, xshift=2mm] (B2) {$B$};
      \draw[myedge] (A2) -- (C2);
      \draw[myedge] (C2) -- (B2);
      \draw[myedge] (A2) -- (B2);
      \node[below=8mm of L2] {b) Type 2};

      \node[mynode, right=21.5mm of C2] (C3) {$C$};
      \node[mynode, below=3.5mm of C3, xshift=5mm] (A3) {$A$};
      \node[mynode3, above right= of A3, xshift=1mm] (L3) {$L$};
      \node[mynode, below right= of L3, xshift=2mm] (B3) {$B$};
      \draw[myedge] (A3) -- (C3);
      \draw[myedge] (B3) -- (C3);
      \draw[myedge] (A3) -- (B3);
      \node[below=8mm of L3] {c) Type 3};
    \end{tikzpicture}
  }
  \caption{\label{triangle2_fig}Possibly learnt latent triangles. For the
    three types, the independent pair of nodes mentioned in
    Rule~\ref{rule1} is $(B,C)$.}
\end{figure}
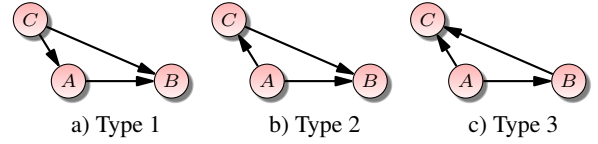

If score-based algorithms never made any mistake in their learning, we
could simply use Rule~\ref{rule1} to identify latent triangles and the
location of the latent confounders. However, in practice, similarly to
constraint-based methods, they do make
mistakes, which may induce our algorithm to incorrectly identify spurious
latent triangles. Fortunately, it is possible to increase the robustness of
the latent triangles identification by not only considering the dependences
between nodes $A$, $B$ and $C$ but also by taking into account those involving node
$D$, the parent of $B$ in $\GGG^*$ (See Fig.~\ref{triangle_fig}.a). 

The dotted arcs in Fig.~\ref{triangle3_fig} are those involving $D$ and
nodes $A,B,C$ that should be learnt by the score-based algorithm. In Type~1
triangles, there should therefore exist an arc $D \rightarrow B$, which can
be translated as ``node $B$ must have at least 3 parents in
$\GGG$''. Triangles of Type~1 are therefore considered as latent
only if they satisfy both this property and Rule~\ref{rule2}.
For Type~3 triangles, the situation is more complex because, in addition to Arc
$D \rightarrow B$, there should also exist an arc between $D$ and $C$
to account for $\xncondindepd{D}{C}{\{A,B\}}{\GGG^*}$. Note that Triangle
$(D,B,C)$ cannot be misinterpreted as latent because both pairs
$(B,C)$ and $(C,D)$ can be made independent given some $\mathbf{Z} \subseteq
\XXX_O$, hence ruling out Rule~\ref{rule1}. Indeed, according to
Fig.~\ref{triangle_fig}.a, $\condindepd{C}{D}{\mathbf{Z}}{\GGG^*}$ for any
$\mathbf{Z} \subseteq \XXX_O \backslash \{C,D\}$ equal to $\emptyset$, $\{A\}$ or $\{B\}$. 
Unfortunately, for the same reason that
Types~1, 2 and 3 triangles encode the same $d$-separation properties, the
orientations of the arcs of Triangle $(D,B,C)$ can be
reversed (provided they do not induce directed cycles). As a consequence,
in the learnt graph $\GGG$, node $D$ may be a child of $B$ rather than its
parent. At first sight, it is difficult to discriminate between the true children of $B$
in $\GGG^*$ and its true parents in which the arcs of Triangle $(D,B,C)$ have been
reversed. However, note that if $D$ were a true child of $B$ in
$\GGG^*$, $\condindepd{C}{D}{\mathbf{Z}}{\GGG^*}$ would hold for
$\mathbf{Z}$ only equal to $\emptyset$ or $\{B\}$ (see
Fig.~\ref{triangle_fig}.a). As a consequence, 
only the true parents $D$ of $B$ in $\GGG$ are such that 
$\condindepd{C}{D}{\mathbf{Z}}{\GGG^*}$ given
some set $\mathbf{Z} \subseteq \XXX_O \backslash \{C,D\}$ containing node $A$.
So Type~3 triangles are considered as latent only if they satisfy both
this property and Rule~\ref{rule2}. Overall, this 
results in Algorithm~\ref{algo1}.

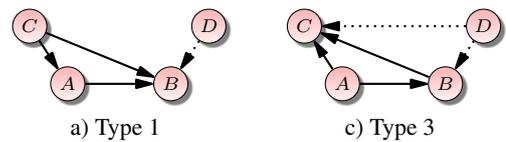
\begin{figure}[b]
  \centerline{
    \begin{tikzpicture}[
      mynode/.style={circle, draw, font=\scriptsize,
        bottom color = red!5, top color = red!30,
        black, circular drop shadow , text = black, 
        inner sep=0.4mm, minimum size=4.5mm,
        node distance=5mm and 2mm},      
      mynode2/.style={circle, dotted, draw, font=\scriptsize,
        bottom color = BlockBodyColor!50, top color = BlockBodyColor,
        black, circular drop shadow , text = black, 
        inner sep=0.4mm, minimum size=4.5mm,
        node distance=5mm and 2mm},      
      mynode3/.style={circle, draw, font=\scriptsize,
        bottom color = white, top color = white,
        white, text = white, 
        inner sep=0.4mm, minimum size=4.5mm,
        node distance=5mm and 2mm},
      myedge/.style={thick, ->,
        >={Stealth[inset=0pt,length=8pt,angle'=28,round]}},
      myedge2/.style={myedge, dotted}
      ]
      \node[mynode] (C1) {$C$};
      \node[mynode, below=3mm of C1, xshift=5mm] (A1) {$A$};
      \node[mynode3, above right= 3mm and 2mm of A1, xshift=1mm] (L1) {$L$};
      \node[mynode, below right= 3mm and 2mm of L1, xshift=2mm] (B1) {$B$};
      \node[mynode, above=3mm of B1, xshift=5mm] (D1) {$D$};
      \draw[myedge] (C1) -- (A1);
      \draw[myedge] (C1) -- (B1);
      \draw[myedge] (A1) -- (B1);
      \draw[myedge2] (D1) -- (B1);
      \node[below=7.5mm of L1] {a) Type 1};
   
      \node[mynode, right=31.5mm of C1] (C3) {$C$};
      \node[mynode, below=3mm of C3, xshift=5mm] (A3) {$A$};
      \node[mynode3, above right= 3mm and 2mm of A3, xshift=1mm] (L3) {$L$};
      \node[mynode, below right= 3mm and 2mm of L3, xshift=2mm] (B3) {$B$};
      \node[mynode, above=3mm of B3, xshift=5mm] (D3) {$D$};
      \draw[myedge] (A3) -- (C3);
      \draw[myedge] (B3) -- (C3);
      \draw[myedge] (A3) -- (B3);
      \draw[myedge2] (D3) -- (B3);
      \draw[myedge2] (D3) -- (C3);
      \node[below=7.5mm of L3] {c) Type 3};
    \end{tikzpicture} 
  }
  \caption{\label{triangle3_fig}Arcs learnt in the vicinity of latent triangles.}
\end{figure}

\begin{algorithm}[htb]
  \caption{\label{algo1}Learning with confounders.}
  \KwIn{Dataset $\DDD$}
  \KwOut{The CPDAG of the learnt CBN}
  $\GGG \leftarrow$ DAG learnt from $\DDD$ by a score-based algorithm\\
  $\mathbf{T} \leftarrow$ the triangles of $\GGG$ satisfying Rule~~\ref{rule2} \\
  \tcp{Get the latent triangles\rule{0pt}{3.5mm}}
  $\mathbf{T}_1 \leftarrow \emptyset$ \tcp{{\small latent triangles of type 1}}
  $\mathbf{T}_3 \leftarrow \emptyset$ \tcp{{\small latent triangles of type 3}} 
  \ForEach{triangle $T = (A,B,C)$ in $\mathbf{T}$}{
    \If{$T$ is of Type 1 and $|\Pa_{\GGG}(B)| \geq 3$}{
      $\mathbf{T}_1 \leftarrow \mathbf{T}_1 \cup T$
    }
    \ElseIf{$T$ is of Type 3 and there exists $D \in
      (\Pa_{\GGG}(B) \backslash \{A\}) \cup (\Ch_{\GGG}(B) \backslash \{C\})$
      such that there exists $\mathbf{Z} \subseteq \XXX_O
      \backslash \{C,D\}$ such that  
      $A \in \mathbf{Z}$ and $D \indep_{P} C | \mathbf{Z}$}{
      $\mathbf{T}_3 \leftarrow \mathbf{T}_3 \cup T$
    }
  }
  \tcp{Recreate the latent variables\rule{0pt}{3mm}}
  \ForEach{triangle $T = (A,B,C)$ in $\mathbf{T}_1 \cup \mathbf{T}_3$}{
    Add a new node $L$ (confounder) to $\GGG$ \\
    Remove from $\GGG$ Arc $A \rightarrow B$ \\
    Remove from $\GGG$ the arc between $B$ and $C$ \\
    Add arcs $L \rightarrow A$ and $L \rightarrow B$ to $\GGG$\\
  }
  $\MMM \leftarrow$ the CPDAG of $\GGG$\\
  
  \KwRet{CPDAG $\MMM$}
\end{algorithm}

Remark that, as mentioned previously, on Line~12, Algorithm~\ref{algo1}
is allowed to remove Arc $A \rightarrow B$ because, when this arc truly exists in
$\GGG^*$ (see Fig.~\ref{indistinguish_fig}.d), based on Dataset $\DDD$, it
is impossible for any learning algorithm to distinguish between $\GGG^*$ and
Fig.~\ref{indistinguish_fig}.e. In such a case, our algorithm deliberately
chooses to return the structure of Fig.~\ref{indistinguish_fig}.e in order
to enforce some Occam razor and to avoid requiring some expert
knowledge to select the right structure.

Rule~\ref{rule2} and Algorithm~\ref{algo1} require the determination of
some set $\mathbf{Z} \subseteq \XXX_O$ such that some pairs of nodes are
conditionally independent given $\mathbf{Z}$. Below, we suggest two
algorithms for this purpose.
The first algorithm to determine Set $\mathbf{Z}$
consists of exploiting the fact that, in I-maps, $d$-separation implies
conditional independences. Hence, applying a $d$-separation analysis on
$\GGG$ using, e.g., \citet{zand-lisk20}'s algorithm, it is possible to get
some $d$-separating set $\mathbf{Z}$ and, consequently, a set inducing a
conditional independence. The following proposition justifies that this
approach can be used\footnote{The asymptotic
  consistency of the BIC score was already known in other cases
  \cite{bouc93,koll-fried09}.}:

\begin{proposition}\label{prop_imap}
  Let $\DDD^*$ be a dataset generated by some distribution $P^*$ over
  $\XXX$ for which there exists a perfect map $\GGG^* = (\XXX,\EEE)$, and
  let $\DDD$ be the projection of $\DDD^*$ over $\XXX_O$. Then, as $|\DDD|
  \rightarrow \infty$, every DAG $\GGG$  maximizing the BIC score over
  $\DDD$ is a  minimal I-map.
\end{proposition}

However, when the database is not ``too'' large and Graph $\GGG$ is not
highly trustworthy, another option is to exploit Algorithm~\ref{algo2} which
adds iteratively to $\mathbf{Z}$ the variable $X$ that allows to reduce the
most the dependence between the pair of nodes $U,V$ until an independence
between $U$ and $V$ is inferred or no independence can be proven. In essence, this is the
approach followed by MIIC \cite{vern-sell-affe-sing-isam17}, except that
MIIC estimates dependences through an information theoretic criterion whereas
we exploit the BIC score:

\begin{algorithm}[t]
  \caption{\label{algo2}Finding a set $\mathbf{Z}$
    s.t.\ $U \indep V | \mathbf{Z}$}
  \KwIn{nodes $U,V$, Dataset $\DDD$, max size $h$ of
    $\mathbf{Z}$, risk level $\alpha$,
    sets $\mathbf{C}$ and $\mathbf{F}$ of compulsory and forbidden
    variables, i.e.,  $\mathbf{Z}$ must satisfy $\mathbf{Z} \supseteq
    \mathbf{C}$ and $\mathbf{Z} \cap \mathbf{F} = \emptyset$}
  \KwOut{A set $\mathbf{Z}$ s.t.\ $U \indep V | \mathbf{Z}$ if such set is found, else False}
  $\mathbf{Z} \leftarrow \mathbf{C}$;
  $\mathbf{F} \leftarrow \mathbf{F} \cup \{U,V\}$\\
  $\delta \leftarrow (\dom{U}|-1) \times (|\dom{V}|-1) \prod_{X \in \mathbf{C}} |\dom{X}|$\\
  \lIf{$|\mathbf{Z}| > h$}{\KwRet{False}}
  \lIf{$f_{BIC}(U,V|\mathbf{Z}) < \chi^2_{\delta}(\alpha)$}{\KwRet{$\mathbf{Z}$}}
  \While{$|\mathbf{Z}| < h$}{
    \mbox{$X = \argmin \{f_{BIC}(U,V|\mathbf{Z} \cup \{Y\}) : Y\! \in \XXX_O \backslash
      (\mathbf{Z} \cup \mathbf{F})\}$}\\
    $\mathbf{Z} \leftarrow \mathbf{Z} \cup \{X\}$\\
    $\delta \leftarrow \delta \times |\dom{X}|$ \\
    \lIf{$f_{BIC}(U,V|\mathbf{Z}) < \chi^2_{\delta}(\alpha)$}{\KwRet{$\mathbf{Z}$}}
  }
  \KwRet{False}
\end{algorithm}

\begin{definition}\label{BIC_G2_def}
  For every pair of variables $U,V$ and every set of variables $\mathbf{Z}$,
  let $f_{BIC}(U,V|\mathbf{Z})$ be defined as:

  \vspace{-4mm}
  
  \begin{equation}\label{eq_g2}
    \!\!f_{BIC}(U,V|\mathbf{Z}) = 2 \times \left(S(U|V,\mathbf{Z}) - S(U|\mathbf{Z}) + \frac{1}{2} \log(|\DDD|)
      \delta\right)
  \end{equation}

  \vspace{-1mm}
  
  \noindent where $\delta = (|\dom{U}|-1) \times (|\dom{V}|-1) \times |\dom{\mathbf{Z}}|$ and
  $\dom{U}$ (resp.\ $\dom{V}$, $\dom{\mathbf{Z}}$) is the domain of
  variable $U$ (resp.\ $V$, $\mathbf{Z}$), and $S(\cdot|\cdot)$ is the BIC score.
\end{definition}

The following proposition justifies Algorithm~\ref{algo2}:

\begin{proposition}\label{BIC_chi2_prop}
  If $U$ and $V$ are independent given a set $\mathbf{Z}$, the formula
  of Eq.~(\ref{eq_g2}) follows a $\chi^2$ distribution of $\delta$ degrees of
  freedom. So, given a risk level $\alpha$,
  $U$ and $V$ are judged independent if the value of Eq.~(\ref{eq_g2}) 
  is lower than the critical value $\chi^2_{\delta}(\alpha)$ of the $\chi^2$ distribution.
\end{proposition}

To conclude this section, we provide below the time complexity of
Algorithm~\ref{algo1}, assuming (as we did in our experiments) that the algorithm
used for determining the conditioning sets $\mathbf{Z}$ required in
Rule~\ref{rule2} and Algorithm~\ref{algo1} is Algorithm~\ref{algo2}:

\begin{proposition}\label{complex_prop}
  Assume that the existence of $d$-separating sets $\mathbf{Z}$ is checked
  with Algorithm~\ref{algo2} with $h$ the maximal size allowed for
  $\mathbf{Z}$ and $f_{BIC}$ defined as Eq.~(\ref{eq_g2}). Let $n$ and $m$
  denote the number of nodes 
  and arcs of $\GGG$ as defined on Line~1 respectively. Let $k$ be the maximum number of parents
  and children of the nodes in $\GGG$. Then the
  time complexity of Algorithm~\ref{algo1} over dataset $\DDD$ is $O(n^2k^3h|\DDD|+m\log n)$.
\end{proposition}


\captionsetup{aboveskip=2mm}

\section{Experiments}\label{expe_sec}

In this section, some experiments on classical benchmark CBNs (child ($|\XXX_O|=20$),
water ($|\XXX_O|=32$), insurance ($|\XXX_O|=27$), alarm ($|\XXX_O|=37$),
barley ($|\XXX_O|=48$)) from the BNLearn Bayes net 
repository\footnote{\url{https://www.bnlearn.com/bnrepository}} 
are performed to highlight the effectiveness of our algorithm to
find latent confounders and to recover CBN structures.

\begin{table*}[t]
  \centering
  \scalebox{.82}{
    \begin{tabular}{@{}c@{\ }*{17}{c}@{}}
      \hline
      \toprule
      & & \multicolumn{6}{c}{\textbf{Algorithm 1}} & \multicolumn{5}{c}{\textbf{MIIC}} & \multicolumn{5}{c}{\textbf{FCI}}\\
      \cmidrule(l){3-8} \cmidrule(l){9-13} \cmidrule(l){14-18}
      CBN
      &$|\DDD|$ &ok  &$\neg$ok  &prec.    &recall &F1         &time   &ok &$\neg$ok &prec.      &recall     &F1         &ok   &$\neg$ok &prec. &recall  &F1\\
      \midrule
      \multirow{5}{*}{child}
      &5000	&0.50	&{\bf0.06}  &{\bf0.89} &0.25  &0.39	  &0.014  &1.40       &0.74       &0.65	      &0.70	  &{\bf0.68}  &{\bf1.76}  & 3.96  &0.31	&{\bf0.88}  &0.46\\
      &10000	&0.76	&{\bf0.18}  &{\bf0.81} &0.38  &0.52	  &0.028  &1.58       &0.62       &0.72	      &0.79	  &{\bf0.75}  &{\bf1.86}  & 2.92  &0.39	&{\bf0.93}  &0.55\\
      &20000	&1.10	&{\bf0.12}  &{\bf0.90} &0.55  &{\bf0.68}  &0.050  &1.34       &0.80       &0.63	      &0.67	  &0.65	      &{\bf1.74}  & 2.12  &0.45	&{\bf0.87}  &0.59\\
      &50000	&1.42	&{\bf0.14}  &{\bf0.91} &0.71  &{\bf0.80}  &0.103  &1.40       &0.68       &0.67	      &0.70	  &0.69	      &{\bf1.74}  & 1.08  &0.62	&{\bf0.87}  &0.72\\
      &100000	&1.44	&{\bf0.14}  &{\bf0.91} &0.72  &{\bf0.80}  &0.190  &1.46       &0.56       &0.72	      &0.73	  &0.73       &{\bf1.72}  & 0.68  &0.72	&{\bf0.86}  &0.78\\
      \midrule                                          
      \multirow{5}{*}{water}                            
      &5000	&0.16	&{\bf1.56}  &0.09      &0.08  &0.09	  &0.017  &{\bf0.34}  &2.76       &{\bf0.11}  &{\bf0.17}  &{\bf0.13}  &{\bf0.34}  & 8.26  &0.04	&{\bf0.17}  &0.06\\
      &10000	&0.44	&{\bf0.04}  &{\bf0.92} &0.22  &0.35	  &0.043  &1.24       &1.68       &0.42	      &0.62	  &{\bf0.50}  &{\bf1.54}  & 5.40  &0.22	&{\bf0.77}  &0.34\\
      &20000	&0.80	&{\bf0.22}  &{\bf0.78} &0.40  &{\bf0.53}  &0.104  &1.40       &1.88       &0.43	      &0.70	  &{\bf0.53}  &{\bf1.60}  & 6.04  &0.21	&{\bf0.80}  &0.33\\
      &50000	&1.16	&{\bf0.60}  &{\bf0.66} &0.58  &{\bf0.62}  &0.132  &1.40       &1.92       &0.42       &0.70	  &0.53	      &{\bf1.56}  & 7.86  &0.17	&{\bf0.78}  &0.27\\
      &100000	&1.22	&{\bf1.04}  &{\bf0.54} &0.61  &{\bf0.57}  &0.248  &1.44       &2.78       &0.34	      &0.72	  &0.46       &{\bf1.52}  & 8.32  &0.15	&{\bf0.76}  &0.26\\
      \midrule                                          
      \multirow{5}{*}{\begin{tabular}{@{}c@{}}insu-\\rance\end{tabular}}
      &5000	&0.20	&{\bf0.20}  &{\bf0.50} &0.10  &0.17	  &0.042  &1.62       &4.04       &0.29	      &0.81	  &{\bf0.42}  &{\bf1.88}  & 8.86  &0.18	&{\bf0.94}  &0.30\\
      &10000	&0.30	&{\bf0.28}  &{\bf0.52} &0.15  &0.23	  &0.079  &1.48       &3.32       &0.31	      &0.74	  &{\bf0.44}  &{\bf1.96}  & 8.78  &0.18	&{\bf0.98}  &0.31\\
      &20000	&0.62	&{\bf0.30}  &{\bf0.67} &0.31  &0.42	  &0.172  &1.66       &3.88       &0.30	      &0.83	  &{\bf0.44}  &{\bf1.94}  & 8.44  &0.19	&{\bf0.97}  &0.31\\
      &50000	&1.16	&{\bf0.54}  &{\bf0.68} &0.58  &{\bf0.63}  &0.373  &1.78       &3.48       &0.34	      &0.89	  &0.49	      &{\bf1.90}  & 8.66  &0.18	&{\bf0.95}  &0.30\\
      &100000	&1.32	&{\bf0.72}  &{\bf0.65} &0.66  &{\bf0.65}  &0.791  &{\bf1.76}  &3.00       &0.37	      &{\bf0.88}  &0.52	      &{\bf1.76}  & 8.20  &0.18	&{\bf0.88}  &0.29\\
      \midrule                                          
      \multirow{5}{*}{alarm}                            
      &5000	&0.36	&{\bf0.40}  &{\bf0.47} &0.18  &0.26	  &0.058  &1.24       &1.48       &0.46	      &0.62	  &{\bf0.53}  &{\bf1.48}  & 4.74  &0.24	&{\bf0.74}  &0.36\\
      &10000	&0.86	&{\bf0.58}  &{\bf0.60} &0.43  &0.50	  &0.119  &1.50       &0.98       &{\bf0.60}  &0.75	  &{\bf0.67}  &{\bf1.56}  & 4.14  &0.27	&{\bf0.78}  &0.41\\
      &20000	&1.12	&0.66	    &0.63      &0.56  &0.59	  &0.180  &1.50       &{\bf0.52}  &{\bf0.74}  &0.75	  &{\bf0.75}  &{\bf1.72}  & 3.42  &0.33	&{\bf0.86}  &0.48\\
      &50000	&1.36	&1.16	    &0.54      &0.68  &0.60	  &0.385  &1.64       &{\bf1.04}  &{\bf0.61}  &0.82	  &{\bf0.70}  &{\bf1.68}  & 2.92  &0.37	&{\bf0.84}  &0.51\\
      &100000	&1.44	&1.18	    &0.55      &0.72  &0.62	  &0.744  &1.62       &{\bf0.76}  &{\bf0.68}  &0.81	  &{\bf0.74}  &{\bf1.70}  & 2.42  &0.41	&{\bf0.85}  &0.56\\
      \midrule                                          
      \multirow{5}{*}{barley}                           
      &5000	&0.02	&{\bf0.00}  &{\bf1.00} &0.01  &0.02	  &0.013  &0.60       &3.14       &0.16	      &0.30	  &{\bf0.21}  &{\bf1.58}  &20.70  &0.07	&{\bf0.79}  &0.13\\
      &10000	&0.08	&{\bf0.00}  &{\bf1.00} &0.04  &0.08	  &0.032  &0.74       &2.30       &0.24	      &0.37	  &{\bf0.29}  &{\bf1.78}  &19.72  &0.08	&{\bf0.89}  &0.15\\
      &20000	&0.36	&{\bf0.00}  &{\bf1.00} &0.18  &0.31	  &0.076  &1.02       &1.84       &0.36	      &0.51	  &{\bf0.42}  &{\bf1.86}  &16.56  &0.10	&{\bf0.93}  &0.18\\
      &50000	&0.74	&{\bf0.00}  &{\bf1.00} &0.37  &{\bf0.54}  &0.229  &1.34       &3.72       &0.26	      &0.67	  &0.38	      &{\bf1.86}  &14.86  &0.11	&{\bf0.93}  &0.20\\
      &100000	&0.96	&{\bf0.08}  &{\bf0.92} &0.48  &{\bf0.63}  &0.511  &1.36       &2.96       &0.31	      &0.68	  &0.43	      &{\bf1.84}  &13.70  &0.12	&{\bf0.92}  &0.21\\
 \bottomrule
  \end{tabular}}
  \caption{\label{v2l2_tab}Confounders found for different dataset sizes
    and CBNs with 2 Boolean confounders.}
\end{table*}

For each experiment,
a CBN is selected, and some new (latent) variables $L_i$ are added to it. To
make them confounders of a semi-Markovian causal model, for each $L_i$, two
nodes of $\XXX_O$ are randomly chosen to become $L_i$'s
children. To fit the propositions of the paper, Line~1 of
Algorithm~\ref{algo1} is performed using the CPBayes 
exact score-based learning algorithm \cite{beek-hoff15}. For this purpose,
Datasets $\DDD$ need to be converted into so-called {\em instances} that
are passed as input to CPBayes. These contain all the possible nodes'
sets that CPBayes will consider as potential parent sets. So, to control
the combinatorial explosion it has to face, CPBayes requires limiting the
number of possible 
parents of the nodes. 
In the experiments, we set this limit to 4 because i)~no node of $\XXX_O$ had more than
4 parents in the original CBN (the one without confounders); and ii)~this 
enabled to control the amount of computations performed by CPBayes (see \cite{beek-hoff15} for
more details). As a consequence, 
since the score-based algorithm may add 2 additional parents ($A$ and $C$)
to some $L_i$'s children ($B$), as shown in Fig.~\ref{triangle_fig}.b, all
the $L_i$'s children are selected randomly but with the constraint that
they have 1 or 2 parents in $\XXX_O$. This upper bound constraint is only due to our
use of CPBayes, not to Algorithm~\ref{algo1}. But the lower bound is due to
Proposition~\ref{prop_indistinguishable1}.

To discover that some nodes $A$ and $B$ are children of some confounder
$L_i$, all the learning algorithms rely in some way or another on the fact
that $A$ and $B$ are conditionnally dependent given any set $\mathbf{Z}$.
This imposes some restrictions on the way the conditional probability
distributions (CPD) of $L_i$ and its children should be generated. Actually,
assume that they are uniformly randomly generated. Then the cells of the joint
distribution $P$ of $A$, $B$ and their parents is a sample generated from a
uniform distribution. Testing the conditional (in)dependence of $A$ and $B$ 
given some sets $\mathbf{Z}$ strictly included in $A$ and $B$'s parents
amounts to marginalize out some variables from $P$ or, equivalently, to sum
some values of $P$. The sum of 2 independent variables uniformly
distributed follows a triangular distribution and, by the central limit
theorem, the sum of more than 2 variables tends to a variable normally
distributed. As such, the values of the marginals of $P$ have much more
chances to be located on the mode of the distribution than on the tails. In
other words, the cells of the marginals of $P$ tend to have more or less
the same values, which makes $A$ and $B$ appear to be independent, and no learning
algorithm can determine that they are the children of a confounder. This is
the reason why the CPDs of $A$, $B$ and $L_i$ need to be generated
differently.

In our experiments, the CPD of $L_i$ is set to a uniform distribution in
order to maximize the chances of $A$ and $B$ to be dependent. For $A$
(resp.\ $B$), for each value of its parents, the CPD of $A$ (resp.\ $B$) is set to 
a mixture of a Dirichlet distribution whose hyperparameters $\alpha_i$ are
all set to 4 and a Dirac distribution. The weight of the latter is selected
randomly between 2/3 and 5/6. Such mixtures tend experimentally to limit the
effect of the central limit theorem but, of course, do not discard it
completely. So, to check whether $A$ and $B$ have some chance to be
identified as dependent, we test whether the values of their mutual
information and their conditional mutual information given their
parents are higher than some thresholds (these are some extreme cases for the
$\mathbf{Z}$ sets mentioned in the preceding paragraph). If this is
the case, the CPDs of $A$ and $B$ are judged admissible for the
experiments. To make them 
realistic, the thresholds are defined as the averages of the mutual
information and condition mutual information respectively of all the pairs of nodes
with a common parent in the original CBN. The use of these
information-theoretic criteria has been made to favor algorithms like MIIC over our algorithm.

For each CBN with confounders created as defined above, a dataset $\DDD^*$
is randomly generated using the pyAgrum 1.13.0 library \cite{duca-gonz-wuil20} and
Dataset $\DDD$ which is given as input of the learning algorithms is the
projection of $\DDD^*$ over $\XXX_O$. Algorithm~\ref{algo2} is exploited
for independence testing and its risk level $\alpha$ and Size $h$
are set to 0.05 and 7 respectively\footnote{Sets $\mathbf{Z}$ may include other nodes
  than just the sets of parents of $U$ or $V$. Hence, for robustness, we
  allowed to almost double their maximal size $h$.}.
In the experiments, we compare 3
learning algorithms: Algorithm~\ref{algo1}, MIIC
\cite{vern-sell-affe-sing-isam17} and FCI \cite{spir-glym-schei00}. For
MIIC, we use the pyAgrum's implementation with the
NML correction to be more accurate. For FCI, we use the python causal-learn
0.1.3.8 package \cite{causallearn}. All the tables below display averages of the
results over 50 CBN/datasets. All the experiments are executed on an
Intel Xeon Gold 5218 CPU with 128GB of RAM.

\begin{table*}[t]
  \centering
  \scalebox{.82}{
    \begin{tabular}{@{}c@{\quad}cc@{\quad}*{4}{r@{\quad}}c@{\quad}*{4}{r@{\quad}}c@{\quad}*{4}{r@{\quad}}@{}}
      \hline
      \toprule
      & & \multicolumn{5}{c}{\textbf{Algorithm 1}} & \multicolumn{5}{c}{\textbf{MIIC}} & \multicolumn{5}{c}{\textbf{FCI}}\\
      \cmidrule(l){3-7} \cmidrule(l){8-12} \cmidrule(l){13-17}
      CBN&$|\DDD|$&ok&miss&rev.~&type&xs~~~&ok&miss&rev.~&type&xs~~~&ok&miss&rev.~&type&xs~~~\\
      \midrule
      \multirow{5}{*}{child}
      &5000	&{\bf22.46}  & 3.08	  &1.02	  &{\bf2.44}  &{\bf2.54}  &19.46       &{\bf 2.00}  &3.18  & 4.36	& 3.42	&15.00	     & 4.96	&{\bf0.22}  & 8.82      & 8.00\\
      &10000	&{\bf23.10}  & 2.52	  &1.52	  &{\bf1.86}  &{\bf3.12}  &20.72       &{\bf 1.18}  &3.28  & 3.82	& 3.42	&15.64	     & 3.40	&{\bf0.06}  & 9.90      & 6.00\\
      &20000	&{\bf22.88}  & 1.84	  &2.22	  &{\bf2.06}  &{\bf3.62}  &19.94       &{\bf 1.44}  &2.84  & 4.78	& 3.82	&16.70	     & 2.60	&{\bf0.00}  & 9.70      & 4.44\\
      &50000	&{\bf22.98}  & 1.22	  &2.84	  &{\bf1.96}  &4.34	  &20.36       &{\bf 1.30}  &2.98  & 4.36	& 3.74	&20.80	     & 1.54	&{\bf0.00}  & 6.66      &{\bf 2.54}\\
      &100000	&{\bf23.20}  &{\bf 1.14}  &2.88	  &{\bf1.78}  &4.50	  &21.28       &{\bf 1.14}  &3.04  & 3.54	& 3.72	&22.38	     & 1.22	&{\bf0.00}  & 5.40      &{\bf 1.88}\\
      \midrule
      \multirow{5}{*}{water}                                                                                                                                                               
      &5000	&{\bf16.72}  &{\bf40.18}  &4.72	  &8.38       &{\bf4.60}  &15.48       &40.20	    &7.26  &{\bf 7.06}	& 9.38	&13.46	     &45.54	&{\bf1.38}  & 9.62      &11.50\\
      &10000	&{\bf19.00}  &{\bf36.98}  &7.58	  &6.44	      &{\bf5.96}  &17.82       &37.78	    &8.04  &{\bf 6.36}  & 9.40	&17.90	     &43.02	&{\bf1.02}  & 8.06      &12.82\\
      &20000	&21.12	     &{\bf36.64}  &6.82	  &5.42	      &{\bf8.72}  &{\bf21.84}  &38.10	    &4.78  &{\bf 5.28}	&11.36	&17.76	     &43.76	&{\bf0.52}  & 7.96      &16.88\\
      &50000	&{\bf26.22}  &{\bf32.48}  &7.20	  &4.10	      &{\bf7.02}  &23.76       &34.70	    &6.12  &{\bf 5.42}	&10.58	&25.12	     &38.66	&{\bf0.76}  & 5.46      &16.42\\
      &100000	&{\bf29.94}  &{\bf29.42}  &7.10	  &{\bf3.54}  &{\bf7.06}  &27.44       &32.32	    &5.34  & 4.90	&12.54	&26.80	     &37.68	&{\bf0.52}  & 5.00      &17.26\\
      \midrule
      \multirow{5}{*}{insurance}                                                                                                                                                           
      &5000	&{\bf32.38}  &14.44	  &2.46	  &{\bf6.72}  &{\bf3.82}  &29.20       &{\bf12.70}  &3.60  &10.50	&12.44	&23.90	     &21.60	&{\bf0.46}  &10.04      &18.06\\
      &10000	&{\bf34.68}  &12.22	  &2.38	  &{\bf6.72}  &{\bf4.06}  &29.86       &{\bf11.70}  &3.52  &10.92	&11.46	&25.76	     &18.64	&{\bf0.30}  &11.30      &17.72\\
      &20000	&{\bf38.76}  &{\bf 9.20}  &3.34	  &{\bf4.70}  &{\bf5.04}  &30.42       &10.00	    &3.74  &11.84	&13.12	&27.98 	     &16.18	&{\bf0.14}  &11.70	&17.08\\
      &50000	&{\bf41.08}  &{\bf 7.58}  &4.42	  &{\bf2.92}  &{\bf6.42}  &31.32       & 8.34	    &4.00  &12.34	&12.94	&29.42	     &13.92	&{\bf0.10}  &12.56	&17.70\\
      &100000	&{\bf40.34}  &{\bf 7.08}  &4.94	  &{\bf3.64}  &{\bf7.48}  &31.46       & 7.88	    &3.88  &12.78	&13.40	&31.36	     &12.64	&{\bf0.12}  &11.88	&16.96\\
      \midrule
      \multirow{5}{*}{alarm}                                                                                                                                                          
      &5000	&38.10	     & 6.42	  &3.90	  &1.58	      &{\bf5.08}  &{\bf40.04}  &{\bf 5.54}  &3.10  &{\bf 1.32}	& 7.72	&37.34	     &10.36	&{\bf0.58}  & 1.72	&10.40\\
      &10000	&39.08	     & 4.52	  &4.96	  &1.44	      &{\bf5.88}  &{\bf40.92}  &{\bf 4.46}  &3.30  & 1.32	& 6.72	&39.48	     & 8.88	&{\bf0.46}  &{\bf 1.18}	& 9.22\\
      &20000	&40.64	     &{\bf 3.08}  &5.02	  &1.26	      &{\bf5.82}  &42.48       & 3.54	    &3.32  &{\bf 0.66}	& 5.92	&{\bf42.90}  & 6.06	&{\bf0.18}  & 0.86	& 7.60\\
      &50000	&40.28	     &{\bf 2.62}  &5.86	  &1.24	      &7.46	  &42.46       & 3.02	    &3.60  & 0.92	& 8.16	&{\bf43.44}  & 5.48	&{\bf0.18}  &{\bf 0.90}	&{\bf 6.58}\\
      &100000	&40.72	     &{\bf 1.84}  &6.02	  &1.42	      &7.72	  &43.04       & 2.44	    &3.66  &{\bf 0.86}	& 7.76	&{\bf44.66}  & 4.18	&{\bf0.10}  & 1.06	&{\bf 5.60}\\
      \midrule
      \multirow{5}{*}{barley}                                                                                                                                                          
      &5000	&41.46	     &34.66 	  &3.02	  &8.86	      &{\bf6.32}  &{\bf51.12}  &{\bf25.14}  &4.32  & 7.42	&13.74	&42.48	     &38.30	&{\bf2.08}  &{\bf 5.14}	&42.46\\
      &10000	&46.30	     &29.10	  &2.82	  &9.78	      &{\bf5.40}  &{\bf56.10}  &{\bf21.52}  &4.44  & 5.94	&13.34	&47.48	     &33.70	&{\bf1.38}  &{\bf 5.44} &40.72\\
      &20000	&52.08	     &22.70	  &3.74	  &9.48	      &{\bf4.64}  &{\bf60.26}  &{\bf17.90}  &4.80  &{\bf 5.04}	&13.44	&53.72	     &27.32	&{\bf1.68}  & 5.28	&34.70\\
      &50000	&57.74	     &17.96	  &4.28	  &8.02	      &{\bf4.54}  &{\bf67.24}  &{\bf12.96}  &4.42  &{\bf 3.38}	&20.72	&59.74	     &21.26	&{\bf1.34}  & 5.66	&31.02\\
      &100000	&62.52	     &14.80	  &5.00	  &5.68	      &{\bf4.48}  &{\bf69.72}  &{\bf12.16}  &3.96  &{\bf 2.16}	&21.76	&62.14	     &19.16	&{\bf1.02}  & 5.68	&29.30\\
      \bottomrule
    \end{tabular}}
  \caption{\label{dag_tab}Comparisons of the learnt CPDAGs with those of the CBNs (with confounders) that generated the datasets.}
\end{table*}

In Table~\ref{v2l2_tab}, to every original CBN, 2 confounders have been
added whose domain sizes are equal to 2. The table compares the performance of
Algorithm~\ref{algo1}, MIIC and FCI for different CBNs with different
dataset sizes. Columns ok and $\neg$ok contain the number of
correctly and wrongly identified confounders respectively. Column
``prec.'' displays the precision metrics, i.e., it is equal to ok / (ok + $\neg$ok). Column
``recall'' is the usual recall metrics, i.e., it is equal to ok / the number of
confounders $L_i$ added to the original CBN. The F1 score is defined as
$2 \times (\mbox{precision} \times \mbox{recall}) / (\mbox{precision} + \mbox{recall})$.
Column ``time'' reports the average computation times in seconds of Lines~2
to 16 of Algorithm~\ref{algo1} (finding the confounders). These lines increase only marginally the
structure learning computation times.

FCI finds most of the latent confounders. This is
the reason why, in Table~\ref{v2l2_tab}, it is the best in terms of {\em recall}. Unfortunately, it also 
misidentifies numerous variables as confounder's children. This is the
reason why its precision and F1 score are, by far, never the best. Here, we should
emphasize that we identify confounders with $\leftrightarrow$ connections, that is,
we do not take into account $\circ$ labels since those express an
uncertainty about the existence of arrow heads. Converting every $-\!\circ$
into $\rightarrow$ results in adding numerous false positive, which would
decrease significantly the quality of the solutions found by FCI.

MIIC identifies
fewer confounders but it also makes much fewer mistakes, which makes it
better in terms of F1 score. Algo.~\ref{algo1} makes even fewer
mistakes and this is the reason why its precision is almost always the best of all the
algorithms. Yet, it is somewhat cautious and tends to miss some
confounders, which explains why its {\em recall} is not the best. However,
the larger the size of the dataset, the higher the number of confounders
correctly identified by Algorithm~\ref{algo1}. This can also be
observed in terms of F1 score: for small datasets, MIIC outperforms
Algorithm~\ref{algo1} but this is the converse when $|\DDD|$
increases. This phenomenon illustrates empirically 
Proposition~\ref{prop_triangle}.

Table~\ref{dag_tab} reports how the CPDAGs learnt by Algorithm~\ref{algo1},
MIIC and FCI compare with those of the CBNs (with their confounders) that
were used to generate the datasets. Columns ``ok'' (resp.\ ``miss'')
indicate the number of arcs and edges that were learnt correctly (resp.\
that existed in the generating CBN but for which no arc nor edge was learnt).
As can be observed, Algorithm~\ref{algo1} and MIIC are the best for these
metrics and their results are quite comparable. 
Of course, for these metrics, the larger the dataset, the
better the quality of the learnt CPDAGs. Column ``rev.'' displays the number
of arcs in the CPDAG of the generating CBN that were learnt in the opposite
direction, i.e., the direction of the causality is incorrectly learnt.
For this metrics, FCI always outperforms the other algorithms. Note, however,
that the number of reversed arcs is always very small for all the CBNs and
all dataset sizes. Column ``type'' refers to the number of arcs (resp.\
edges) that were learnt as edges (resp.\ arcs), i.e., their types (directed,
undirected) are incorrect. For this metrics, Algorithm~\ref{algo1} and MIIC
usually outperform FCI and, in general, when Algorithm~\ref{algo1}
outperforms MIIC, the difference is bigger than when this is the
converse. Finally, Column ``xs'' reports the number of arcs or edges that
belong to the learnt CPDAG but whose extremal nodes are linked neither by
an arc nor by an edge in the generating CBN. Here, Algorithm~\ref{algo1}
significantly outperforms both MIIC and FCI. This means that
Algorithm~\ref{algo1} is less prone to learn spurious direct
causes. Overall, empirically, Algorithm~\ref{algo1} is very competitive and often
produces CPDAGs closer to those of the original CBNs 
than the other two methods


\section{Conclusion and Perspectives}\label{conclu_sec}

In this paper, we have introduced the first score-based CBN structure learning
algorithm for discrete variables that searches only the space of DAGs,
exploits only observational data and, yet, is capable of identifying
some latent confounders. It has been justified mathematically, notably through
Proposition~\ref{prop_triangle}. In addition, theoretically, it is asymptotically
guaranteed to produce an I-map.
Experiments highlighted it effectiveness, especially for large
datasets. Notably, both in terms of the CPDAGs and the latent confounders found,
the results of this algorithm may be judged as very
competitive compared to those of its constraint-based competitors like MIIC or
FCI.

For future works, to be more scalable, we plan to substitute the use of the
CPBayes on Line~1 of Algorithm~\ref{algo1} by a faster approximate
algorithm like greedy hill climbing (GHC). Usually, GHC-like methods
make more mistakes in the directions of the causal arcs learned than
CPBayes. So, to compensate for this issue, the rules used in Lines~6
and 8 may certainly have to be improved. Notably, in
Algorithm~\ref{algo1}, we did not take into account Type~2 triangles (see
Figure~\ref{triangle2_fig}.b) because, empirically, they were very seldom
encountered in the experiments. However, with GHC-like algorithms, this may
not be the case anymore and they should be taken into account.

Perhaps a more immediate improvement could be made by observing that,
whenever Type~3 triangles are found, their $A \rightarrow C$ connection
is in the wrong direction (see Figure~\ref{triangle2_fig}.b). Therefore, to
produce better CPDAGs, Algorithm~\ref{algo1} should reverse this arc and relearn the
neighborhood of node $C$. In addition, as shown in Propositions~\ref{prop_indistinguishable1}
and \ref{prop_indistinguishable2}, some structures are indistinguishable
from the observational data point of view and, among all the structures of
Fig.~\ref{indistinguish_fig},  Algorithm~\ref{algo1} makes
the decision to select only those of Fig.~\ref{indistinguish_fig}.c and
\ref{indistinguish_fig}.e. So, in the same spirit as PAGs, the output of
Algorithm~\ref{algo1} could be improved to express the uncertainty about
which of these structures should be selected. Note that, in this case, PAG's
$\circ$ labels are not sufficient since the uncertainty not only concerns
the location of arrow heads but also the very existence of some arcs (like
$C \rightarrow B$ of Fig.~\ref{indistinguish_fig}.e) for which no
independence test can detect that they can be dispensed with.

\newpage


\bibliography{article}

\begin{thebibliography}{32}
\providecommand{\natexlab}[1]{#1}
\providecommand{\url}[1]{\texttt{#1}}
\expandafter\ifx\csname urlstyle\endcsname\relax
  \providecommand{\doi}[1]{doi: #1}\else
  \providecommand{\doi}{doi: \begingroup \urlstyle{rm}\Url}\fi

\bibitem[Bouckaert(1993)]{bouc93}
R.~R. Bouckaert.
\newblock Probabilistic network construction using the minimum description
  length principle.
\newblock In \emph{Proc. of the European conference on symbolic and
  quantitative approaches to reasoning and uncertainty (ECSQARU'93)}, pages
  41--48, 1993.

\bibitem[Chickering(1995)]{chic95}
D.~Chickering.
\newblock A transformational characterization of equivalent {B}ayesian network
  structures.
\newblock In \emph{Proc. of UAI}, pages 87--98, 1995.

\bibitem[Chickering(2002)]{chic02}
M.~Chickering.
\newblock Optimal structure identification with greedy search.
\newblock \emph{Journal of Machine Learning Research}, 3:\penalty0 507--554,
  2002.

\bibitem[Colombo and Maathuis(2014)]{colo-maat14}
D.~Colombo and M.~Maathuis.
\newblock Order-independent constraint-based causal structure learning.
\newblock \emph{Journal of Machine Learning Research}, 15:\penalty0 3921--3962,
  2014.

\bibitem[Colombo et~al.(2012)Colombo, Maathuis, Kalisch, and
  Richardson]{colo-maat-kali-rich12}
D.~Colombo, M.~Maathuis, M.~Kalisch, and T.~Richardson.
\newblock Learning high-dimensional directed acyclic graphs with latent and
  selection variables.
\newblock \emph{The Annals of Statistics}, 40\penalty0 (1):\penalty0 294--321,
  2012.

\bibitem[Cooper and Herskovits(1992)]{coop-hers92}
G.~Cooper and E.~Herskovits.
\newblock A {B}ayesian method for the induction of probabilistic networks form
  data.
\newblock \emph{Machine Learning}, 9\penalty0 (4):\penalty0 309--347, 1992.

\bibitem[Ducamp et~al.(2020)Ducamp, Gonzales, and Wuillemin]{duca-gonz-wuil20}
G.~Ducamp, C.~Gonzales, and P.-H. Wuillemin.
\newblock {aGrUM/pyAgrum}: a toolbox to build models and algorithms for
  probabilistic graphical models in python.
\newblock In \emph{Proc. of the International Conference on Probabilistic
  Graphical Models (PGM'20)}, pages 609--612, 2020.

\bibitem[Heckerman et~al.(1995)Heckerman, Geiger, and
  Chickering]{heck-geig-chic95}
D.~Heckerman, D.~Geiger, and D.~Chickering.
\newblock Learning {B}ayesian networks: The combination of knowledge and
  statistical data.
\newblock \emph{Machine Learning}, 20:\penalty0 197--243, 1995.

\bibitem[Koller and Friedman(2009)]{koll-fried09}
D.~Koller and N.~Friedman.
\newblock \emph{Probabilistic Graphical Models: Principles and Techniques}.
\newblock MIT Press, 2009.

\bibitem[Lemeire et~al.(2012)Lemeire, Meganck, Cartella, and
  Liu]{leme-mega-cart-liu12}
J.~Lemeire, S.~Meganck, F.~Cartella, and T.~Liu.
\newblock Conservative independence-based causal structure learning in absence
  of adjacency faithfulness.
\newblock \emph{International Journal of Approximate Reasoning}, 53\penalty0
  (9):\penalty0 1305--1325, 2012.

\bibitem[Mabrouk et~al.(2014)Mabrouk, Gonzales, Jabet-Chevalier, and
  Chojnaki]{mabr-gonz-jabe-choj14}
A.~Mabrouk, C.~Gonzales, K.~Jabet-Chevalier, and E.~Chojnaki.
\newblock An efficient {B}ayesian network structure learning algorithm in the
  presence of deterministic relations.
\newblock In \emph{Proc. of the European Conference on Artificial Intelligence
  (ECAI'14)}, pages 567--572, 2014.

\bibitem[Meek(1995)]{meek95}
C.~Meek.
\newblock Causal inference and causal explanation with background knowledge.
\newblock In \emph{Proc. of the Conference on Uncertainty in Artiﬁcal
  Intelligence (UAI'95)}, pages 403--410, 1995.

\bibitem[Ogarrio et~al.(2016)Ogarrio, Spirtes, and Ramsey]{ogar-spir-rams16}
J.~Ogarrio, P.~Spirtes, and J.~Ramsey.
\newblock A hybrid causal search algorithm for latent variable models.
\newblock In \emph{Proc. of International Conference on Probabilistic Graphical
  Models (PGM'16)}, pages 368--379, 2016.

\bibitem[Pearl(1988)]{pear88}
J.~Pearl.
\newblock \emph{Probabilistic Reasoning in Intelligent Systems: Networks of
  Plausible Inference}.
\newblock Morgan Kaufman, 1988.

\bibitem[Pearl(2009)]{pear09}
J.~Pearl.
\newblock \emph{Causality}.
\newblock Cambridge University Press, 2nd edition, 2009.

\bibitem[Ramsey et~al.(2006)Ramsey, Spirtes, and Zhang]{rams-spir-zhan06}
J.~Ramsey, P.~Spirtes, and J.~Zhang.
\newblock Adjacency-faithfulness and conservative causal inference.
\newblock In \emph{Proc. of the Conference on Uncertainty in Artificial
  Intelligence (UAI'06)}, pages 401--408, 2006.

\bibitem[Richardson and Spirtes(1998)]{rich-spir98}
T.~Richardson and P.~Spirtes.
\newblock Scoring ancestral graph models.
\newblock Technical Report CMU-PHIL-98, Carnegie Mellon, 1998.

\bibitem[Richardson and Spirtes(2002)]{rich-spir02}
T.~Richardson and P.~Spirtes.
\newblock Ancestral graph markov models.
\newblock \emph{Annals of Statistics}, 30\penalty0 (4):\penalty0 962--1030,
  2002.

\bibitem[Schwarz(1978)]{schw78}
G.~Schwarz.
\newblock Estimating the dimension of a model.
\newblock \emph{Annals of Statistics}, 6:\penalty0 461--464, 1978.

\bibitem[Spirtes and Glymour(1991)]{spir-glym91}
P.~Spirtes and C.~Glymour.
\newblock An algorithm for fast recovery of sparse causal graphs.
\newblock \emph{Social Science Computer Review}, 9\penalty0 (1):\penalty0
  62--72, 1991.

\bibitem[Spirtes et~al.(2000)Spirtes, Glymour, and Scheines]{spir-glym-schei00}
P.~Spirtes, C.~Glymour, and R.~Scheines.
\newblock \emph{Causation, Prediction, and Search}.
\newblock MIT press, 2nd edition, 2000.

\bibitem[Tian and Pearl(2002)]{tian-pear02}
J.~Tian and J.~Pearl.
\newblock On the identification of causal effects.
\newblock Technical Report R-290-L, UCLA C.S. Lab, 2002.

\bibitem[Triantafillou and Tsamardinos(2016)]{trian-tsam16}
S.~Triantafillou and I.~Tsamardinos.
\newblock Score based vs constraint based causal learning in the presence of
  confounders.
\newblock In \emph{Proc. of the ``Causation: Foundation to Application''
  Workshop, Uncertainty in Artificial Intelligence}, 2016.

\bibitem[Tr\"osser et~al.(2022)Tr\"osser, {de}~Givry, and
  Katsirelos]{troes-givr-kats22}
F.~Tr\"osser, S.~{de}~Givry, and G.~Katsirelos.
\newblock "structured set variable domains in {B}ayesian network structure
  learning.
\newblock In \emph{Proc. of Principles and Practice of Constraint Programming
  (CP'22)}, pages 37:1--37:9, 2022.

\bibitem[Tsamardinos et~al.(2003)Tsamardinos, Aliferis, and
  Statnikov]{tsam-alif-stat03}
I.~Tsamardinos, C.~Aliferis, and S.~Statnikov.
\newblock Algorithms for large scale {M}arkov blanket discovery.
\newblock In \emph{Proc. of the international FLAIRS conference (FLAIRS'03)},
  pages 376--381, 2003.

\bibitem[{v}an Beek and Hoffmann(2015)]{beek-hoff15}
P.~{v}an Beek and H.-F. Hoffmann.
\newblock Machine learning of {B}ayesian networks using constraint programming.
\newblock In \emph{Proc. of Principles and Practice of Constraint Programming
  (CP'15)}, pages 429--445, 2015.

\bibitem[{v}an Beek and Lee(2017)]{beek-lee17}
P.~{v}an Beek and C.~Lee.
\newblock An experimental analysis of anytime algorithms for {B}ayesian network
  structure learning.
\newblock \emph{Proceedings of Machine Learning Research}, 73:\penalty0 69--80,
  2017.

\bibitem[{v}an~{d}er Zander and Li\'skiewicz(2020)]{zand-lisk20}
B.~{v}an~{d}er Zander and M.~Li\'skiewicz.
\newblock Finding minimal {d}-separators in linear time and applications.
\newblock In \emph{Proc. of the Conference on Uncertainty in Artificial
  Intelligence Conference (UAI'20)}, pages 637--647, 2020.

\bibitem[Verma and Pearl(1990)]{verm-pear90}
T.~Verma and J.~Pearl.
\newblock Equivalence and synthesis of causal models.
\newblock In \emph{Proc. of the Conference on Uncertainty in Artiﬁcal
  Intelligence (UAI'90)}, pages 220--227, 1990.

\bibitem[Verny et~al.(2017)Verny, Sella, Affeldt, Singh, and
  Isambert]{vern-sell-affe-sing-isam17}
L.~Verny, N.~Sella, S.~Affeldt, P.~Singh, and H.~Isambert.
\newblock Learning causal networks with latent variables from multivariate
  information in genomic data.
\newblock \emph{PLOS Computational Biology}, 2017.

\bibitem[Zhang(2008)]{zhan08}
J.~Zhang.
\newblock On the completeness of orientation rules for causal discovery in the
  presence of latent confounders and selection bias.
\newblock \emph{Artificial Intelligence}, 172\penalty0 (16):\penalty0
  1873--1896, 2008.

\bibitem[Zheng et~al.(2023)Zheng, Huang, Chen, Ramsey, Gong, Cai, Shimizu,
  Spirtes, and Zhang]{causallearn}
Y.~Zheng, B.~Huang, W.~Chen, J.~Ramsey, M.~Gong, R.~Cai, S.~Shimizu,
  P.~Spirtes, and K.~Zhang.
\newblock Causal-learn: causal discovery in python.
\newblock \emph{arXiv preprint arXiv:2307.16405}, 2023.

\end{thebibliography}

\newpage

\section*{Supplementary Material}

\appendix

\section{Proofs}

Below are the proofs of all the propositions of the paper.

\begin{lemma}\label{lemma_parent}
  Let $\XXX_O$ and $\XXX_H$ be two disjoint sets of random variables and
  let $\XXX = \XXX_O \cup \XXX_H$.
  Let $\DDD^*$ be a dataset generated by some distribution $P^*$ over
  $\XXX$ for which there exists a perfect map $\GGG^* = (\XXX,\EEE)$, and
  let $\DDD$ be the projection of $\DDD^*$ over $\XXX_O$,
  i.e., the dataset resulting from the removal from $\DDD^*$ of all the
  values of the variables of $\XXX_H$. Let $\GGG$ be a
  DAG maximizing the BIC score over $\DDD$. Then, as $|\DDD| \rightarrow
  \infty$, if, in Graph $\GGG$, there is no arc between a pair of nodes $(A,B)$
  and no directed path\footnote{A directed path $\CCC = \langle
    X_1,\ldots,X_k \rangle$ is a trail $\CCC$ such that, for all $i \in
    \{1,\ldots,k-1\}$, the graph contains Arc $X_i \rightarrow X_{i+1}$.} from $B$ to $A$, 
  then $A \indep_{P^*} B | \Pa_{\GGG}(B)$.
\end{lemma}

\begin{proof}{Lemma~\ref{lemma_parent}}
  As there is no directed path from $B$ to $A$ in $\GGG$, adding $A
  \rightarrow B$ cannot create a directed cycle, hence this would produce a
  new  DAG. If $\GGG$ does not contain 
  this arc, this therefore means that the BIC score $S(B | \Pa_{\GGG}(B) \cup \{A\})$
  is lower than or equal to $S(B | \Pa_{\GGG}(B))$. $\DDD$ is the
  projection of $\DDD^*$ over $\XXX_O$, so the distribution $P$ which
  generated $\DDD$ is the projection of $P^*$ over $\XXX_O$, hence, for any
  $\mathbf{Z},\mathbf{W} \subseteq \XXX_O$, $P(\mathbf{Z} | \mathbf{W}) =
  P^*(\mathbf{Z} | \mathbf{W})$. As a consequence, as $|\DDD| \rightarrow \infty$, the
  BIC score can be expressed in terms of mutual information  ($\III_{P^*}$)
  and entropy ($\HHH_{P^*}$) over $P^*$ \citep[p792]{koll-fried09}:
  \begin{displaymath}
    \begin{array}{l}
      S(B | \Pa_{\GGG}(B) \cup \{A\}) \rightarrow |\DDD| \times
      [\III_{P^*}(B;\Pa_{\GGG}(B) \cup \{A\}) \\
      \hspace*{2cm}-\  \displaystyle\HHH_{P^*}(B)] -
      \frac{\log |\DDD|}{2} dim(B|\Pa_{\GGG}(B) \cup \{A\}),
    \end{array}
  \end{displaymath}
  where, for all $Z$, $\mathbf{W}$:
  \begin{displaymath}
    \renewcommand{\extrarowheight}{3mm}
    \begin{array}{@{}l@{}}
      \displaystyle dim(B|\Pa_{\GGG}(B) \cup \{A\}) = (|\dom{B}|-1) \times |\dom{A}|
      \times \!\! \prod_{X \in \Pa_{\GGG}(B)} |\dom{X}|, \\
      \displaystyle \III_{P^*}(Z;\mathbf{W}) \! = \sum_{z \in \dom{Z}} \sum_{\mathbf{w} \in \dom{\mathbf{W}}}
      \!\! P^*(z,\mathbf{w}) \log
      \left(\frac{P^*(z,\mathbf{w})}{P^*(z)P^*(\mathbf{w})}\right), \\
     \HHH_{P^*}(Z) = \sum_{z \in \dom{Z}} P^*(z) \log (P^*(z)).
    \end{array}
  \end{displaymath}
  So, if $\alpha$ denotes the difference between the two scores,
  we have that:
  \begin{eqnarray*}
    \alpha
    & =
    & S(B | \Pa_{\GGG}(B) \cup \{A\}) - S(B | \Pa_{\GGG}(B)) \\
    & =
    & |\DDD| \times \left[\III_{P^*}(B;\Pa_{\GGG}(B) \cup \{A\})
      - \III_{P^*}(B;\Pa_{\GGG}(B)) \right] \\
    & & -\ \delta \log |\DDD| / 2,
  \end{eqnarray*}
  with $\delta = dim(B|\Pa_{\GGG}(B) \cup \{A\}) - dim(B|\Pa_{\GGG}(B))$. When
  $|\DDD| \rightarrow \infty$, $\log |\DDD|$ is infinitely smaller than $|\DDD|$. Hence
  $\alpha \leq 0$ if and only if $\III_{P^*}(B;\Pa_{\GGG}(B) \cup \{A\}) \leq
  \III_{P^*}(B;\Pa_{\GGG}(B))$. But, for every probability distribution $Q$ and every
  $X,Y,\mathbf{Z}$, it always holds that  $\III_{Q}(X;\mathbf{Z} \cup \{Y\}) \geq
  \III_{Q}(X;\mathbf{Z})$, with equality only if $X \indep_{Q} Y |
  \mathbf{Z}$. Hence $\alpha \leq 0$ if and only if
  $\III_{P^*}(B;\Pa_{\GGG}(B) \cup \{A\}) = \III_{P^*}(B;\Pa_{\GGG}(B))$.
  This implies that $A \indep_{P^*} B | \Pa_{\GGG}(B)$. 
\end{proof}


\begin{corollary}\label{coro_edge}
  Let $\XXX_O$ and $\XXX_H$ be two disjoint sets of random variables and
  let $\XXX = \XXX_O \cup \XXX_H$.
  Let $\DDD^*$ be a dataset generated by some distribution $P^*$ over
  $\XXX$ for which there exists a perfect map $\GGG^* = (\XXX,\EEE)$, and
  let $\DDD$ be the projection of $\DDD^*$ over $\XXX_O$. Let $\GGG$ be a
  DAG maximizing the BIC score over $\DDD$. Then, as $|\DDD| \rightarrow
  \infty$, for any $A,B \in \XXX_O$, if $\GGG^*$ contains Arc
  $A \rightarrow B$, then Graph $\GGG$ contains
  either Arc $A \rightarrow B$ or $B \rightarrow A$.
\end{corollary}

\begin{proof}{Corollary~\ref{coro_edge}}
  Assume that, for some pair $A,B \in \XXX_O$ such that there exists an arc
  between $A$ and $B$ in $\GGG^*$, Graph $\GGG$ contains neither Arc $A
  \rightarrow B$ nor Arc $B \rightarrow A$. It is impossible that $\GGG$
  contains both a directed path from $A$ to $B$ and another one from $B$ to
  $A$ because the concatenation of these paths would be a directed
  cycle. Without loss of generality, assume that there exists no directed
  path from $B$ to $A$ in $\GGG$, then, by Lemma~\ref{lemma_parent},
  $A \indep_{P^*} B | \Pa_{\GGG}(B)$. As $\GGG^*$ is a perfect map, this
  implies that $\condindepd{A}{B}{\Pa_{\GGG}(B)}{\GGG^*}$, a contradiction
  since $\GGG^*$ contains an arc between $A$ and $B$.
\end{proof}


\begin{proof}{Proposition~\ref{prop_triangle}}
  Let $L$ be a variable in $\XXX_H$ and let $A,B$ be
  its children. Let $\GGG$ be a graph maximizing the BIC score over $\DDD$.
  First, let us prove that there must exist an arc between $A$ and $B$ in
  $\GGG$. Assume the contrary. It is impossible that there exists both a
  directed path $\CCC_{AB}$ from $A$ to $B$ in $\GGG$ and a directed path
  $\CCC_{BA}$ from $B$ to $A$ in $\GGG$ because their concatenation would
  be a directed cycle, a contradiction since $\GGG$ is a DAG. Without loss
  of generality, assume that there exists no directed path from $B$ to $A$ in
  $\GGG$. Then, as, by hypothesis, there exists no arc between
  $A$ and $B$, by Lemma~\ref{lemma_parent}, $A \indep_{P^*} B | \Pa_{\GGG}(B)$.
  But this is impossible because, whatever the set $\mathbf{Z} \subseteq
  \XXX_O \backslash \{A,B\}$, Trail $\langle A, L, B \rangle$ is active
  given $\mathbf{Z}$ in $\GGG^*$, so that $\ncondindepd{A}{B}{\Pa_{\GGG}(B)}{\GGG^*}$,
  which contradicts $A \indep_{P^*} B | \Pa_{\GGG}(B)$ since $\GGG^*$ is a
  perfect map. So, there must exist an arc between $A$ and $B$ in
  $\GGG$. In the rest of the proof, without loss of generality, assume
  that it contains Arc $A \rightarrow B$. 

  Let $C \in \Pa_{\GGG^*}(A) \cap \XXX_O$. Such a $C$ exists by the
  hypotheses at the beginning of the Proposition. By
  Corollary~\ref{coro_edge}, $\GGG$ contains an arc between $A$ 
  and $C$. As for an arc between $B$ and $C$, assume that there exists none
  in $\GGG$. Note that Trail $\CCC_{CB} = \langle C,A,L,B \rangle$ is
  always active in $\GGG^*$ given any $\mathbf{Z} \subseteq \XXX_O
  \backslash \{B,C\}$ that contains $A$.
  So $B \notindep_{P^*} C | \mathbf{Z}$ since $\GGG^*$ is a
  perfect map. Now, as shown at the beginning of the proof, there exists
  either i)~no directed path from $B$ to $C$ in $\GGG$; or ii)~no
  directed path from $C$ to $B$ in $\GGG$. If Case~i) obtains, by
  Lemma~\ref{lemma_parent}, $C \indep_{P^*} B | \Pa_{\GGG}(B)$, which is
  impossible since $A \in \Pa_{\GGG}(B)$ (according to the preceding
  paragraph), so that Trail $\CCC_{CB}$ is active in $\GGG^*$ given $\mathbf{Z}
  = \Pa_{\GGG}(B)$, a contradiction (since it is equivalent to  $C
  \notindep_{P^*} B | \Pa_{\GGG}(B)$). If Case~ii) obtains (but not
  Case~i)), then the arc added previously between $A$
  and $C$ could not be $C \rightarrow A$ because, since Case~i) does not
  obtain, there exists a directed path from $B$ to $C$, hence also a
  directed path from $A$ to $C$ (since $\GGG$ contains Arc $A \rightarrow
  B$). So Arc $C \rightarrow A$ would create a directed cycle. 
  Hence, $A \in \Pa_{\GGG}(C)$. Now, by Lemma~\ref{lemma_parent},
  $B \indep_{P^*} C | \Pa_{\GGG}(C)$, which is impossible since
  $\mathbf{Z} = \Pa_{\GGG}(C) \supseteq \{A\}$ would make Trail $\CCC_{CB}$
  active in $\GGG^*$, a contradiction. As a consequence, there must
  necessarily exist an arc between $B$ and $C$ in $\GGG$. So, overall,
  $\GGG$ contains clique $(A,B,C)$.
\end{proof}


\begin{proof}{Proposition~\ref{prop_indistinguishable1}}
  $\condindepd{\mathbf{U}}{\mathbf{V}}{\mathbf{W}}{\GGG}$ is equivalent to
  $\condindepd{U}{V}{\mathbf{W}}{\GGG}$ for all $U \in \mathbf{U}$ and $V
  \in \mathbf{V}$. So, we just need to prove that, for any $U,V \in
  \XXX_O$, $U \neq V$, $\condindepd{U}{V}{\mathbf{W}}{\GGG}
  \Longleftrightarrow \condindepd{U}{V}{\mathbf{W}}{\GGG^*}$.
  
  First, note that $\GGG^*$ does not contain Arc $B \rightarrow A$,
  otherwise $\Pa_{\GGG^*}(A)$ would contain $B$. Adding Arc $A \rightarrow B$
  cannot create a directed cycle in $\GGG^*$ because $A$ has only one
  parent, $L$, which cannot be involved in any cycle since it has no parent.  
  Let $\desc_{\GGG}(A)$ denote the set of descendants of $A$ in $\GGG$,
  i.e., $\desc_{\GGG}(A) = \{X:$ there exists a directed path from $A$ to $X\}$.
  Note first that $\desc_{\GGG}(A) = \desc_{\GGG^*}(A) \cup
  \desc_{\GGG^*}(B)$ and, for all $X \in \XXX \backslash \{A,L\}$, we have
  that $\desc_{\GGG}(X) = \desc_{\GGG^*}(X)$.

  \newpage

  Now, consider any simple\footnote{A simple trail is a trail in which no
    node appears more than once. It is well known that, for any active
    (resp.\ blocked) trail between a pair of nodes $(X,Y)$, there also exists a
    simple active (resp.\ blocked) trail between $X$ and $Y$.}
  trail $\CCC^* = \langle X_1=U,\ldots,X_k=V \rangle$
  between $U$ and $V$ in $\GGG^*$. Assume it is blocked in $\GGG^*$. This
  is the case if and only if i)~it
  contains is a convergent connection at some node $T$ such that neither $T$
  nor its descendants are in $\mathbf{W}$; or ii)~it contains a non-convergent
  connection at some node $T \in \mathbf{W}$. 

  Assume Trail $\CCC^*$ does not contain $L$, then it also exists in
  $\GGG$. Case i) cannot occur at $T=A$ because $A$ has only one parent in
  $\GGG^*$ or at $T=L$ because $\CCC^*$ does not contain $L$ by
  hypothesis. Since all the other nodes in $\GGG^*$ have the same set of
  descendants as their counterparts in $\GGG$, the convergent connections
  in $\CCC^*$ have therefore exactly the same status (blocked, active) in
  $\GGG$ and $\GGG^*$. As for case ii), any non-convergent connection in
  $\CCC^*$ is the same in $\GGG$ and $\GGG^*$ and cannot involve $L$ (by
  hypothesis). Hence, they have the same status in $\GGG$ and $\GGG^*$.
  So $\CCC^*$ is blocked or active in both $\GGG$ and $\GGG^*$.

  Now, if $\CCC^*$ contains $L$, it includes the non-convergent connection
  $A \leftarrow L \rightarrow B$. Let $\CCC$ be the trail of $\GGG$
  obtained from $\CCC^*$ by substituting $A \leftarrow L
  \rightarrow B$ by $A \rightarrow B$. For the same reasons as above, all
  the nodes different from $A,L$ have the same connections in $\CCC$ and
  $\CCC^*$ and the same set of descendants. Hence their status (blocked,
  active) are the same in $\CCC$ and $\CCC^*$. Node $A$ cannot have a
  convergent connection in $\CCC^*$ because $\Pa_{\GGG^*}(A) = \{L\}$. In
  $\CCC$, its connection is also non-convergent since its child $B$ in
  $\GGG$ is its neighbor in $\CCC$. So, the connection at $A$ is
  non-convergent and has the same status in both $\CCC$ and $\CCC^*$.
  Node $L$ has a non-convergent connection in
  $\CCC^*$ but, as it is unobserved, it cannot belong to $\mathbf{W}$ and
  cannot block the trail. So, removing it from $\CCC^*$ cannot change the
  status of the trail. Hence, overall, all the trails in $\GGG^*$ can be
  mapped into a trail in $\GGG$ with exactly the same status. 

  Conversely, let $\CCC$ be a simple trail in $\GGG$. If this trail does
  not include Arc $A \rightarrow B$, then it also belongs to $\GGG^*$ and the
  same reasoning as above shows that this trail has the same status in
  $\GGG$ and $\GGG^*$.

  If, on the other hand, $\CCC$ contains Arc $A
  \rightarrow B$, then substituting it by $A \leftarrow L \rightarrow B$
  results in a new trail $\CCC^*$ of $\GGG^*$. As above, all the nodes
  except $A,B,L$ have the same status. In addition, $L$ cannot block trail $\CCC^*$ and
  the connection at $A$ is non-convergent in both $\CCC$ and $\CCC^*$
  (because it has at most one parent, $L$). If the connection at $B$ is
  convergent in $\CCC$, it is also convergent in $\CCC^*$, with its parent $A$
  substituted by $L$. If the connection at $B$ in $\CCC$ is
  non-convergent, this means that it is of the form $A \rightarrow B
  \rightarrow X$, and in $\CCC^*$, its connection is $L \rightarrow B
  \rightarrow X$, also a non-convergent connection. Hence, overall, 
  every trail of $\GGG$ can be mapped into a trail of $\GGG^*$ with the same status.

  So, for all $U \in \mathbf{U}$ and $V \in \mathbf{V}$, we have that
  $\condindepd{U}{V}{\mathbf{W}}{\GGG}
  \Longleftrightarrow \condindepd{U}{V}{\mathbf{W}}{\GGG^*}$.
\end{proof}


\begin{proof}{Proposition~\ref{prop_indistinguishable2}}
  As for Proposition~\ref{prop_indistinguishable1},  we just need to prove
  that, for any $U,V \in \XXX_O$, $U \neq V$, $\condindepd{U}{V}{\mathbf{W}}{\GGG}
  \Longleftrightarrow \condindepd{U}{V}{\mathbf{W}}{\GGG^*}$.

  A simple trail $\CCC^* = \langle X_1=U,\ldots,X_k=V \rangle$
  between $U$ and $V$ is active in $\GGG^*$ if and only if i)~for all the
  nodes $T$ with a convergent connection, either $T$ or some of its
  descendants are in $\mathbf{W}$; and ii)~nodes $T$ with non-convergent
  connections do not belong to $\mathbf{W}$.
  Note that, by definition of $\GGG$, if $\desc_{\GGG}(X)$ denotes the set
  of descendants of Node $X$ in $\GGG$, then $\desc_{\GGG}(X) = \desc_{\GGG^*}(X)$
  for any $X \in \XXX \backslash \{L\}$.
  Let us first show that, if there exists an active trail in $\GGG^*$, then
  there also exists an active trail in $\GGG$.
  
  Let $\CCC^*$ be an active trail of $\GGG^*$ that does not contain $L$,
  then it also exists in $\GGG$ and, since $\desc_{\GGG}(X) =
  \desc_{\GGG^*}(X)$ for all $X \in \XXX \backslash \{L\}$, it is also an
  active trail of $\GGG$.

  Assume now that Trail $\CCC^*$ contains $L$. Node $L$ is equal neither
  to $X_1$ nor to $X_k$ since $U$ and $V$ belong to $\XXX_O$. So, let $i
  \neq 1,k$ be the index such that, in Trail $\CCC^*$, $X_i = L$. By definition of $\GGG^*$, it
  holds that $\{X_{i-1},X_{i+1}\} = \{A,B\}$. Without loss of generality,
  assume below that $X_{i-1}=A$ and $X_{i+1}=B$ (if this is the converse,
  consider the reversed trail $\langle Y_1=V,\ldots,Y_k=U \rangle$, which
  has the same $d$-separation status as $\CCC^*$, i.e., it is an active trail).

  If $X_1 = A$, then $A$ is the first node of Trail $\CCC^*$. Let $\CCC =
  \langle X_1=A, X_3 =B, X_4, \ldots, X_k \rangle$. Then $\CCC$ belongs to both
  $\GGG^*$ and $\GGG$. In addition, the types of connection
  (convergent/non-convergent) of all the nodes
  $X_3,\ldots,X_{k-1}$ are the same in $\CCC^*$ and $\CCC$, and the sets of
  descendants of these nodes are the same in $\GGG^*$ and $\GGG$. Finally, $A$
  and $X_{k}$ cannot belong to $\mathbf{W}$ since they belong to
  $\mathbf{U}$ and $\mathbf{V}$ respectively. Hence, $\CCC$ is an
  active trail in $\GGG$.

  Suppose now that $X_1 \neq A$. Then, Trail $\CCC^*$ contains Node
  $X_{i-2}$, i.e., it contains Subsequence $\langle X_{i-2}, A, L, B
  \rangle$. If $X_{i-2}$ is a child of $A$, then Trail $\CCC = \langle
  X_1, \ldots, X_{i-2}, X_{i-1} = A, X_{i+1} = B, \ldots, X_k \rangle$
  belongs to both $\GGG^*$ and $\GGG$. In addition, all the nodes in $\CCC$
  have the same type of connection (convergent/non-convergent) as in
  $\CCC^*$. Therefore, since $\desc_{\GGG}(X) = \desc_{\GGG^*}(X)$ for all
  the nodes $X \in \CCC$, if $\CCC^*$ is an active trail of $\GGG^*$, then
  $\CCC$ is an active trail of $\GGG$.

  Assume now that $X_{i-2}$ is a parent of $A$. Let
  $\CCC = \langle X_1, \ldots, X_{i-2},  X_{i+1} = B, \ldots, X_k \rangle$.
  $\CCC$ is a trail of $\GGG$. In addition, all of its nodes have the same
  type of connection (convergent/non-convergent) in $\GGG$ as their corresponding
  node of $\CCC^*$ in $\GGG^*$. Therefore, since $\desc_{\GGG}(X) =
  \desc_{\GGG^*}(X)$ for all the nodes $X \in \CCC$, $\CCC$ is an active
  trail of $\GGG$.

  To complete the proof, let us now show that, if there exists an active
  trail $\CCC = \langle X_1=U,\ldots,X_k=V \rangle$ in $\GGG$, then there
  also exists an active trail $\CCC^*$ in $\GGG^*$.

  If $B \not\in \CCC$,
  then $\CCC$ also belong to $\GGG^*$ 
  because the only arcs that belong to $\GGG$ but not to $\GGG^*$ are Arcs
  $Y \rightarrow B$ with $Y \in \Pa_{\GGG}(A)$. As $\desc_{\GGG}(X) =
  \desc_{\GGG^*}(X)$ for all the nodes $X \in \CCC$, $\CCC$ is also an active
  trail in $\GGG^*$. Assume now that $B \in \CCC$. For the same reason as
  above, if none of the neighbors of $B$ in $\CCC$ belong to
  $\Pa_{\GGG}(A)$, then $\CCC$ is a trail of $\GGG^*$ and is also
  active.

  Let $i$ be the index such that $X_i = B$. If some neighbors of
  $B$ belong to $\Pa_{\GGG}(A)$, two cases can obtain: case~1)~exactly one
  neighbor of $B$ belongs to $\Pa_{\GGG}(A)$; and case~2) exactly two
  neighbors of $B$ belongs to $\Pa_{\GGG}(A)$. As for case~1), without loss
  of generality, assume that $X_{i-1} \in \Pa_{\GGG}(A)$ (else, reverse
  Trail $\CCC$). If $A \not\in \mathbf{W}$, then Trail
  $\CCC^* = \langle X_1,\ldots,X_{i-1},A,X_{i}=B,\ldots,X_k \rangle$
  belongs to $\GGG^*$. In addition, the type of connection
  (convergent/non-convergent) of all the nodes in $\CCC^*$ except
  $A$ are the same as their counterpart in $\CCC$. The connection at $A$ is
  non-convergent and $A \not\in \mathbf{W}$, so that $A$ does not block
  $\CCC^*$ in $\GGG^*$. So, overall $\CCC^*$ is an active trail of
  $\GGG^*$. If, now, $A \in \mathbf{W}$, then let 
  $\CCC^* = \langle X_1,\ldots,X_{i-1},A,L,X_{i}=B,\ldots,X_k
  \rangle$. This is a trail of $\GGG^*$ for which the type of connection
  (convergent/non-convergent) of all the nodes in $\CCC^*$ except
  $A$ and $L$ are the same as their counterpart in $\CCC$. The connection
  at $A$ is convergent and $A \in \mathbf{W}$, so $A$ does not block
  $\CCC^*$ in $\GGG^*$. The connection at $L$ is non-convergent and $L
  \not\in \mathbf{W}$ since $L \not\in \XXX_O$. So $L$ does not block
  $\CCC^*$. Consequently, $\CCC^*$ is active in $\GGG^*$.

  Consider now Case~2), i.e., the two neighbors of $B$ in $\CCC$ belong to
  $\Pa_{\GGG}(A)$. This means that $X_{i-1},B,X_{i+1}$ form a convergent
  connection. Since $\CCC$ is active, either $B$ or some of its descendants
  belong to $\mathbf{W}$. But $B$ is a child of $A$ in both $\GGG$ and
  $\GGG^*$. Therefore, either $A$ or some of its descendants
  belong to $\mathbf{W}$. As a consequence, if $\CCC^*$ is the trail
  obtained from $\CCC$ by substituting $B$ by $A$, then the connection at
  $A$ in $\CCC^*$ is convergent and $A$ does not block $\CCC^*$. For all
  the other nodes, the connections are similar in $\CCC$ and $\CCC^*$. Hence,
  $\CCC^*$ is active in $\GGG^*$. This completes the proof.
\end{proof}


\begin{proof}{Proposition~\ref{prop_imap}}
  Lemma~\ref{lemma_parent} considers pairs of nodes $(A,B)$ such that there
  exists no directed path from $B$ to $A$. So, $A$ is not a descendant of
  $B$. In addition, in this lemma, there exists no arc between $A$ and $B$,
  so $A$ is not a parent of $B$. So Lemma~\ref{lemma_parent} states that,
  in $\GGG$, every node is independent of its non-descendants given its
  parents.

  Let $\GGG$ be a DAG maximizing the BIC score. Then it contains no arc $X
  \rightarrow Y$ such that the graph $\GGG'$ resulting from the removal of
  $X \rightarrow Y$ from $\GGG$ can maximize the BIC score. Indeed,
  as $|\DDD| \rightarrow \infty$, Score $S(Y |\Pa_{\GGG}(Y)) \rightarrow
  |\DDD| \times \left[\III_{P^*}(Y;\Pa_{\GGG'}(Y) \cup \{X\}) - 
    \HHH_{P^*}(Y) \right] - \frac{\log |\DDD|}{2} dim(B|\Pa_{\GGG'}(Y) \cup \{X\})$
  (see the proof of Lemma~\ref{lemma_parent}).
  It cannot be the case that $\III_{P^*}(Y;\Pa_{\GGG'}(Y) \cup \{X\}) <
  \III_{P^*}(Y;\Pa_{\GGG'}(Y))$ else $\GGG$ would not maximize the BIC
  score (since the term in $\log |\DDD|$ is infinitely smaller than
  $|\DDD|$). But, for every probability distribution $Q$ and every
  $U,V,\mathbf{Z}$, it always holds that  $\III_{Q}(U;\mathbf{Z} \cup \{V\}) \geq
  \III_{Q}(U;\mathbf{Z})$, with equality only if $U \indep_{Q} V |
  \mathbf{Z}$. So, necessarily, $\III_{P^*}(Y;\Pa_{\GGG'}(Y) \cup \{X\}) =
  \III_{P^*}(Y;\Pa_{\GGG'}(Y))$. As a consequence, we have that:
  \begin{displaymath}
    \renewcommand{\extrarowheight}{3mm}
    \begin{array}{@{}l@{}}
      S(Y |\Pa_{\GGG}(Y)) - S(Y |\Pa_{\GGG'}(Y)) \\
      \hspace*{5mm}\approx
      \displaystyle \frac{\log |\DDD|}{2} [dim(B|\Pa_{\GGG'}(Y) \cup \{X\}) -
      dim(B|\Pa_{\GGG'}(Y))] \\
      \hspace*{5mm}\approx
      \displaystyle \frac{\log |\DDD|}{2} (|\dom{X}|-1) \times (|\dom{Y}|-1)
      \times (|\dom{\Pa_{\GGG'}(Y)}|) > 0.
    \end{array}
  \end{displaymath}
  Hence $\GGG'$ cannot maximize the BIC score. So $\GGG$ is minimal.
  Overall, by Corollary~4, p.~120, of \citep{pear88}, $\GGG$ is a minimal I-map.
\end{proof}


\begin{proof}{Proposition~\ref{BIC_chi2_prop}}
  Let $N_{uv\mathbf{w}}$ denote the number of records in $\DDD$ such that
  $U=u$, $V=v$ and $\mathbf{W}=\mathbf{w}$ and let $N_{u\mathbf{w}} =
  \sum_{v \in \dom{V}} N_{uv\mathbf{w}}$, $N_{v\mathbf{w}} =
  \sum_{u \in \dom{U}} N_{uv\mathbf{w}}$ and $N_{\mathbf{w}} =
  \sum_{u \in \dom{U}} \sum_{v \in \dom{V}} N_{uv\mathbf{w}}$.

  Let $dim(U|V,\mathbf{W})$ and $dim(U|\mathbf{W})$ denote the number of
  free parameters in the conditional probability tables $P(U|V,\mathbf{W})$
  and $P(U|\mathbf{W})$ respectively, i.e., 
  $dim(U|V,\mathbf{W}) = (|\dom{U}|-1) \times |\dom{V}| \times |\dom{\mathbf{W}}|$ and
  $dim(U|\mathbf{W}) = (|\dom{U}|-1) \times |\dom{\mathbf{W}}|$.
  So, we have that:
  \begin{displaymath}
    \Delta_{dim} = \frac{1}{2} \log(|\DDD|) [dim(U|V,\mathbf{W}) -
    dim(U|\mathbf{W})] = \frac{1}{2} \log(|\DDD|) \delta,
  \end{displaymath}
  with $\delta = (|\dom{U}|-1) \times (|\dom{V}|-1) \times |\dom{\mathbf{W}}|$.
  As a consequence, we have that:
  \begin{displaymath}
    \begin{array}{@{}l@{}}
      S(U|V,\mathbf{W}) - S(U|\mathbf{W}) \\
      \hspace*{5mm} = \displaystyle
      \sum_{u,v,\mathbf{w}} N_{uv\mathbf{w}}
      \left[
      \log\left(\frac{N_{uv\mathbf{w}}}{N_{v\mathbf{w}}}\right) -
      \log\left(\frac{N_{u\mathbf{w}}}{N_{\mathbf{w}}}\right)
      \right] - \Delta_{dim} \\
       \hspace*{5mm} = \displaystyle
      \sum_{u,v,\mathbf{w}} N_{uv\mathbf{w}}
      \log\left(\frac{N_{uv\mathbf{w}}N_{\mathbf{w}}}{N_{u\mathbf{w}}N_{v\mathbf{w}}}\right)
      - \Delta_{dim}
    \end{array}
  \end{displaymath}
  So we have that:
  \begin{displaymath}
    \begin{array}{l}
      2 \times \left( S(U|V,\mathbf{W}) - S(U|\mathbf{W}) + \Delta_{dim} \right) \\
      \hspace*{5mm}= \displaystyle 2 \sum_{u,v,\mathbf{w}} N_{uv\mathbf{w}}
      \log\left(\frac{N_{uv\mathbf{w}}N_{\mathbf{w}}}{N_{u\mathbf{w}}N_{v\mathbf{w}}}\right)
      = G2(U,V|\mathbf{W}),
    \end{array}
  \end{displaymath}
  where $G2()$ is the formula used in the classical G-test. It is well-known
  that, when $U$ and $V$ are independent given a set $\mathbf{Z}$, the G2 formula
  follows a $\chi^2$ distribution of $\delta$ degrees of freedom. So, given a risk level $\alpha$,
  $U$ and $V$ are judged independent if the value of $f_{BIC}(U,V|\mathbf{Z})$
  is lower than the critical value $\chi^2_{\delta}(\alpha)$ of the $\chi^2$ distribution.
\end{proof}


\begin{proof}{Proposition~\ref{complex_prop}}
  Let $n$ and $m$ denote the number of nodes and arcs of $\GGG$
  respectively. Let $k$ be the maximum number of parents and children of the
  nodes. Assume that, for conditional independence tests, Algorithm~\ref{algo2}
  is used and that we limit it to sets  $\mathbf{Z}$ such that $|\mathbf{Z}| \leq
  h$. Then Algorithm~\ref{algo2} completes in $O(nh|\DDD|)$ time. Indeed,
  each call to Function $f_{BIC}$ requires parsing the database once, which is
  performed in $O(|\DDD|)$ time. As Step~6 examines all the possible $Y$ nodes,
  its time complexity is $O(n|\DDD|)$. Finally, the {\em while} loop of
  Lines~5--9 is executed at most $h$ times. So, overall, the time
  complexity of Algorithm~\ref{algo2} is $O(nh|\DDD|)$.

  Now, let us determine the complexity of Algorithm~\ref{algo1}.
  In Graph $\GGG$, there exist at most $nk^2$ triangles (for each node,
  triangles can be created by selecting two parents). To check whether one
  is latent (Rule~2), each of its 3 arcs must be examined by Algorithm~\ref{algo2}.
  So Line~2 is completed in $O(n^2k^2h|\DDD|)$ time.
  On Line~2, the information about which pair of nodes is
  independent is cached. So, determining their types on Lines 6 and 8 is
  performed in $O(1)$ time. Determining whether $|\Pa_{\GGG}(X_3)| \geq 3$ can
  also be done in $O(1)$. 
  As for Line~8, there are at most $2k$ nodes $D$ to examine (at most $k$
  parents and $k$ children), each using
  Algorithm~\ref{algo2}. So, the time complexity of Line~8 is
  $O(nkh|\DDD|)$ and, therefore, that of Loop~5--9 is $O(n^2k^3h|\DDD|)$.
  There are at most $nk^2$ triangles in $\GGG$, hence $O(nk^2)$ iterations
  of the {\em for} loop of Lines~10--14. Each instruction on lines 11 to 14
  can be performed in $O(1)$ times, hence the loop of Lines~10--14 can be
  performed in $O(n^2k^3h|\DDD|)$ time. Finally, on Line~15, $\GGG$ is
  transformed into a CPDAG, which can be done in time $O(m\log n)$, see
  \cite{chic95}.

  Overall, the time complexity of Algorithm~\ref{algo1} is therefore
  $O(n^2k^3h|\DDD|+m\log n)$. 
\end{proof}

\section{Additional experiments}

In this section, additional experiments generated similarly to those of Section~\ref{expe_sec}
are performed, highlighting the robustness of our results.
In the first two subsections, we vary the number of latent confounders as well
as their domain size. The experiments highlight the fact that the results
presented in the paper (both recovering the structure and detecting latent
confounders) are not sensitive to these features. 
In the third subsection, we vary the limit on the number of parents required
by CPBayes from 4 to 6. The results show that increasing these numbers 
lead to significant improvements neither in the learning of structures nor
in that of latent confounders.

\subsection{Detecting multiple/non-Boolean latent confounders}

Table~\ref{nblat_lat_tab} reports the learning of the confounders in
datasets generated by the Insurance CBN to which we added 2 to 4 latent
Boolean confounders. As for the results presented in the paper, whatever
the number of latent confounders to be found, Algorithm~1 outperforms MIIC
and FCI w.r.t.\ the wrongly identified confounders and the precision
metrics. FCI outperforms the other algorithms w.r.t.\ the number of
correctly identified latent confounders and the recall metrics. As for the
F1 metrics, Algorithm~1 and MIIC outperform the other algorithms on large
and small datasets respectively.

\begin{table*}[ht]
  \centering
  \scalebox{1}{
    \begin{tabular}{@{}c@{\ }*{17}{c}@{}}
      \hline
      \toprule
      & & \multicolumn{5}{c}{\textbf{Algorithm 1}} & \multicolumn{5}{c}{\textbf{MIIC}} & \multicolumn{5}{c}{\textbf{FCI}}\\
      \cmidrule(l){3-8} \cmidrule(l){9-13} \cmidrule(l){14-18}
      $|\XXX_H|$
      &$|\DDD|$ &ok     &$\neg$ok   &prec.      &recall &F1   &time   &ok    &$\neg$ok &prec.   &recall     &F1         &ok   &$\neg$ok &prec. &recall  &F1\\
      \midrule
      \multirow{5}{*}{2}
      &5000	&0.20	&{\bf0.20}  &{\bf0.50}  &0.10   &0.17	    &0.042  &1.62       &4.04  &0.29	 &0.81	    &{\bf0.42}	&{\bf1.88}	& 8.86	&0.18	&{\bf0.94}  &0.30\\
      &10000	&0.30	&{\bf0.28}  &{\bf0.52}  &0.15   &0.23	    &0.079  &1.48       &3.32  &0.31	 &0.74	    &{\bf0.44}	&{\bf1.96}	& 8.78	&0.18	&{\bf0.98}  &0.31\\
      &20000	&0.62	&{\bf0.30}  &{\bf0.67}  &0.31   &0.42	    &0.172  &1.66       &3.88  &0.30	 &0.83	    &{\bf0.44}	&{\bf1.94}	& 8.44	&0.19	&{\bf0.97}  &0.31\\
      &50000	&1.16	&{\bf0.54}  &{\bf0.68}  &0.58   &{\bf0.63}  &0.373  &1.78       &3.48  &0.34	 &0.89	    &0.49	&{\bf1.90}	& 8.66	&0.18	&{\bf0.95}  &0.30\\
      &100000	&1.32	&{\bf0.72}  &{\bf0.65}  &0.66   &{\bf0.65}  &0.791  &{\bf1.76}  &3.00  &0.37	 &{\bf0.88} &0.52	&{\bf1.76}	& 8.20	&0.18	&{\bf0.88}  &0.29\\
      \midrule
      \multirow{5}{*}{3}
      & 5000	&0.18	&{\bf0.16}  &{\bf0.53}	&0.06	&0.11	    &0.047  &2.04       &5.20  &0.28	 &0.68	    &{\bf0.40}  &{\bf2.62}	& 9.84	&0.21	&{\bf0.87}  &0.34\\
      &10000	&0.48	&{\bf0.20}  &{\bf0.71}	&0.16	&0.26	    &0.088  &2.38       &5.00  &0.32	 &0.79	    &{\bf0.46}  &{\bf2.86}	&10.10	&0.22	&{\bf0.95}  &0.36\\
      &20000	&0.96	&{\bf0.48}  &{\bf0.67}	&0.32	&{\bf0.43}  &0.207  &2.08       &6.22  &0.25	 &0.69	    &0.37	&{\bf2.84}	& 9.90	&0.22	&{\bf0.95}  &0.36\\
      &50000	&1.50	&{\bf0.52}  &{\bf0.74}	&0.50	&{\bf0.60}  &0.469  &2.26       &5.96  &0.27	 &0.75	    &0.40	&{\bf2.78}	& 8.78	&0.24	&{\bf0.93}  &0.38\\
      &100000   &1.74	&{\bf0.76}  &{\bf0.70}	&0.58	&{\bf0.63}  &1.383  &2.28	&5.24  &0.30	 &0.76	    &0.43	&{\bf2.60}	& 8.30	&0.24	&{\bf0.87}  &0.37\\
      \midrule
      \multirow{5}{*}{4}
      & 5000	&0.20	&{\bf0.24}  &{\bf0.45}	&0.05	&0.09	    &0.057  &2.62       &6.66  &0.28	 &0.66	    &{\bf0.39}  &{\bf3.56}	&12.60	&0.22	&{\bf0.89}  &0.35\\
      &10000	&0.40	&{\bf0.32}  &{\bf0.56}	&0.10	&0.17	    &0.104  &2.70       &6.08  &0.31	 &0.68	    &{\bf0.42}  &{\bf3.66}	&11.24	&0.25	&{\bf0.92}  &0.39\\
      &20000	&1.08	&{\bf0.34}  &{\bf0.76}	&0.27	&0.40	    &0.221  &3.04       &6.04  &0.33	 &0.76	    &{\bf0.46}  &{\bf3.70}	&11.04	&0.25	&{\bf0.93}  &0.39\\
      &50000	&1.96	&{\bf0.96}  &{\bf0.67}	&0.49	&{\bf0.57}  &0.588  &2.94       &6.72  &0.30	 &0.74	    &0.43	&{\bf3.68}	&10.20	&0.27	&{\bf0.92}  &0.41\\
      &100000   &2.16	&{\bf1.20}  &{\bf0.64}	&0.54	&{\bf0.59}  &1.748  &2.90	&6.72  &0.30	 &0.73	    &0.43	&{\bf3.38}	& 8.74	&0.28	&{\bf0.85}  &0.42\\
      \bottomrule
    \end{tabular}}
  \caption{\label{nblat_lat_tab}Confounders learnt for Insurance with different numbers of Boolean latent confounders.}
\end{table*}

Table~\ref{dom_lat_tab} reports the learning of the confounders in
datasets generated by the Insurance CBN to which we added two latent
confounders with different domain sizes. Here again, the results are similar to those provided in the paper.

\begin{table*}[ht]
  \centering
  \scalebox{1}{
    \begin{tabular}{@{}c@{\ }*{17}{c}@{}}
      \hline
      \toprule
      & & \multicolumn{5}{c}{\textbf{Algorithm 1}} & \multicolumn{5}{c}{\textbf{MIIC}} & \multicolumn{5}{c}{\textbf{FCI}}\\
      \cmidrule(l){3-8} \cmidrule(l){9-13} \cmidrule(l){14-18}
      $|\dom{L_i}|$
      &$|\DDD|$ &ok  &$\neg$ok  &prec.    &recall &F1         &time   &ok &$\neg$ok &prec.      &recall     &F1         &ok   &$\neg$ok &prec. &recall  &F1\\
      \midrule
      \multirow{5}{*}{2}
      &5000	&0.20	&{\bf0.20}  &{\bf0.50}  &0.10   &0.17	    &0.042  &1.62       &4.04   &0.29  &0.81	   &{\bf0.42}  &{\bf1.88}  & 8.86  &0.18  &{\bf0.94}  &0.30\\
      &10000	&0.30	&{\bf0.28}  &{\bf0.52}  &0.15   &0.23	    &0.079  &1.48       &3.32   &0.31  &0.74	   &{\bf0.44}  &{\bf1.96}  & 8.78  &0.18  &{\bf0.98}  &0.31\\
      &20000	&0.62	&{\bf0.30}  &{\bf0.67}  &0.31   &0.42	    &0.172  &1.66       &3.88   &0.30  &0.83	   &{\bf0.44}  &{\bf1.94}  & 8.44  &0.19  &{\bf0.97}  &0.31\\
      &50000	&1.16	&{\bf0.54}  &{\bf0.68}  &0.58   &{\bf0.63}  &0.373  &1.78       &3.48   &0.34  &0.89	   &0.49       &{\bf1.90}  & 8.66  &0.18  &{\bf0.95}  &0.30\\
      &100000	&1.32	&{\bf0.72}  &{\bf0.65}  &0.66   &{\bf0.65}  &0.791  &{\bf1.76}  &3.00   &0.37  &{\bf0.88}  &0.52       &{\bf1.76}  & 8.20  &0.18  &{\bf0.88}  &0.29\\      
      \midrule
      \multirow{5}{*}{3}
      & 5000	&0.10	&{\bf0.20}  &{\bf0.33}  &0.05   &0.09       &0.048  &1.18       &5.16	&0.19  &0.59	   &{\bf0.28}  &{\bf1.86}  &10.38  &0.15  &{\bf0.93}  &0.26\\
      &10000	&0.26	&{\bf0.28}  &{\bf0.48}  &0.13   &0.20       &0.091  &1.14       &3.48	&0.25  &0.57	   &{\bf0.34}  &{\bf1.74}  & 9.22  &0.16  &{\bf0.87}  &0.27\\
      &20000	&0.62	&{\bf0.48}  &{\bf0.56}  &0.31   &{\bf0.40}  &0.181  &1.50       &4.46	&0.25  &0.75	   &0.38       &{\bf1.68}  & 8.86  &0.16  &{\bf0.84}  &0.27\\
      &50000	&0.88	&{\bf0.62}  &{\bf0.59}  &0.44   &{\bf0.50}  &0.420  &1.48       &4.16	&0.26  &0.74	   &0.39       &{\bf1.68}  & 8.82  &0.16  &{\bf0.84}  &0.27\\
      &100000   &1.10	&{\bf0.78}  &{\bf0.59}	&0.55   &{\bf0.57}  &1.108  &1.32       &4.46	&0.23  &0.66	   &0.34       &{\bf1.54}  & 7.18  &0.18  &{\bf0.77}  &0.29\\
      \midrule
      \multirow{5}{*}{4}   	 	 	 	 	 	 	 	 	 	 	 	 	 	 
      & 5000	&0.20	&{\bf0.24}  &{\bf0.45}  &0.10   &0.16       &0.051  &1.34	&5.46	&0.20  &0.67	   &{\bf0.30}  &{\bf1.86}  &11.32  &0.14  &{\bf0.93}  &0.25\\
      &10000	&0.28	&{\bf0.26}  &{\bf0.52}  &0.14   &0.22       &0.097  &1.42	&4.80	&0.23  &0.71	   &{\bf0.35}  &{\bf1.82}  &10.10  &0.15  &{\bf0.91}  &0.26\\
      &20000	&0.50	&{\bf0.36}  &{\bf0.58}  &0.25   &0.35       &0.197  &1.42	&4.58	&0.24  &0.71	   &{\bf0.36}  &{\bf1.66}  & 9.74  &0.15  &{\bf0.83}  &0.25\\
      &50000	&0.84	&{\bf0.64}  &{\bf0.57}  &0.42   &{\bf0.48}  &0.442  &1.32	&5.50	&0.19  &0.66	   &0.30       &{\bf1.64}  & 8.82  &0.16  &{\bf0.82}  &0.26\\
      &100000   &1.08	&{\bf0.60}  &{\bf0.64}	&0.54   &{\bf0.59}  &1.121  &1.52	&6.50	&0.19  &0.76	   &0.30       &{\bf1.62}  & 8.44  &0.16  &{\bf0.81}  &0.27\\
 \bottomrule
    \end{tabular}}
  \caption{\label{dom_lat_tab}Confounders learnt for Insurance with latent confounders with different domain sizes.}
\end{table*}

\subsection{The quality of the CPDAGs}

Table~\ref{nblat_dag_tab} compares the CPDAGs learnt by Algorithm~1, MIIC and FCI with
those of the CBNs that generated the datasets. These CBNs correspond to Insurance to
which we added 2 to 4 Boolean latent confounders. As for the results
presented in the paper, Algorithm~1 outperforms MIIC and FCI in terms of
the number of arcs/edges learnt correctly as well as in terms of the
incorrect types of edges/arcs (undirected edges (resp.\ arcs) of the generating CBN
learnt as arcs (resp.\ edges)) and in terms of the edges/arcs learnt in
excess (the original CBN contains neither an edge nor an arc for the pairs
of nodes concerned). FCI outperforms the other algorithms in terms of arcs
reversed. MIIC is the best in terms of missed arcs/edges.

\begin{table*}[ht]
  \centering
  \scalebox{1}{
    \begin{tabular}{@{}c@{\quad}cc@{\quad}*{4}{r@{\quad}}c@{\quad}*{4}{r@{\quad}}c@{\quad}*{4}{r@{\quad}}@{}}
      \hline
      \toprule
      & & \multicolumn{5}{c}{\textbf{Algorithm 1}} & \multicolumn{5}{c}{\textbf{MIIC}} & \multicolumn{5}{c}{\textbf{FCI}}\\
      \cmidrule(l){3-7} \cmidrule(l){8-12} \cmidrule(l){13-17}
      $|\XXX_H|$&$|\DDD|$&ok&miss&rev.~&type&xs~~~&ok&miss&rev.~&type&xs~~~&ok&miss&rev.~&type&xs~~~\\
      \midrule
      \multirow{5}{*}{2}
      &5000	&{\bf32.38}  &14.44	  &2.46	  &{\bf6.72}  &{\bf 3.82}  &29.20       &{\bf12.70}  &3.60  &10.50  &12.44  &23.90       &21.60  &{\bf0.46}  &10.04      &18.06\\
      &10000	&{\bf34.68}  &12.22	  &2.38	  &{\bf6.72}  &{\bf 4.06}  &29.86       &{\bf11.70}  &3.52  &10.92  &11.46  &25.76       &18.64  &{\bf0.30}  &11.30      &17.72\\
      &20000	&{\bf38.76}  &{\bf 9.20}  &3.34	  &{\bf4.70}  &{\bf 5.04}  &30.42       &10.00	     &3.74  &11.84  &13.12  &27.98       &16.18  &{\bf0.14}  &11.70	 &17.08\\
      &50000	&{\bf41.08}  &{\bf 7.58}  &4.42	  &{\bf2.92}  &{\bf 6.42}  &31.32       & 8.34	     &4.00  &12.34  &12.94  &29.42       &13.92  &{\bf0.10}  &12.56	 &17.70\\
      &100000	&{\bf40.34}  &{\bf 7.08}  &4.94	  &{\bf3.64}  &{\bf 7.48}  &31.46       & 7.88	     &3.88  &12.78  &13.40  &31.36       &12.64  &{\bf0.12}  &11.88	 &16.96\\
      \midrule
      \multirow{5}{*}{3}
      & 5000	&{\bf33.52}  &15.46	  &1.70	  &{\bf7.32}  &{\bf 3.70}  &31.94	&{\bf12.08}  &5.24  & 8.74  &17.16  &27.34       &21.44  &{\bf1.60}  & 7.62	 &20.08\\
      &10000	&{\bf34.72}  &13.38	  &2.54	  &{\bf7.36}  &{\bf 4.86}  &33.80	&{\bf10.00}  &5.74  & 8.46  &17.38  &29.62       &19.08  &{\bf1.20}  & 8.10	 &20.62\\
      &20000	&{\bf37.84}  &10.08	  &3.36	  &{\bf6.72}  &{\bf 6.88}  &32.64	&{\bf 9.14}  &5.58  &10.64  &21.94  &31.94       &16.08  &{\bf0.84}  & 9.14	 &20.12\\
      &50000	&{\bf40.42}  & 8.54	  &4.28	  &{\bf4.76}  &{\bf 8.36}  &35.26	&{\bf 7.80}  &4.92  &10.02  &23.24  &34.00       &13.18  &{\bf0.72}  &10.10	 &18.20\\
      &100000   &{\bf41.76}  &{\bf7.66}	  &4.54	  &{\bf4.04}  &{\bf10.10}  &35.70	&{\bf 7.66}  &4.12  &10.52  &21.32  &35.84       &11.82  &{\bf0.42}  & 9.92	 &17.74\\
      \midrule
      \multirow{5}{*}{4}
      & 5000	&{\bf33.86}  &16.08	  &1.84	  &8.22	      &{\bf 4.34}  &33.66	&{\bf11.28}  &6.56  & 8.50  &20.42  &29.10       &23.06  &{\bf1.80}  &{\bf6.04}  &25.46\\
      &10000	&34.44       &14.52	  &2.36	  &8.68	      &{\bf 5.72}  &{\bf35.54}	&{\bf 9.96}  &6.36  & 8.14  &21.82  &32.80       &19.30  &{\bf1.54}  &{\bf6.36}  &22.80\\
      &20000	&{\bf38.54}  &11.78	  &3.40	  &{\bf6.28}  &{\bf 7.74}  &36.70	&{\bf 8.40}  &5.70  & 9.20  &22.66  &33.66       &17.00  &{\bf0.64}  & 8.70	 &22.62\\
      &50000	&{\bf40.34}  & 9.50	  &4.68	  &{\bf5.48}  &{\bf11.74}  &36.04	&{\bf 7.98}  &5.98  &10.00  &26.80  &36.60       &14.04  &{\bf0.48}  & 8.88	 &21.30\\
      &100000   &37.84	     & 8.94	  &5.38	  &7.84	      &{\bf13.60}  &37.94	&{\bf 7.34}  &6.96  & 7.76  &29.02  &{\bf39.92}	 &12.28	 &{\bf1.04}  &{\bf6.76}	 &19.04\\
      \bottomrule
    \end{tabular}}
  \caption{\label{nblat_dag_tab}Comparisons of the learnt CPDAGs with those
    of the CBNs (with confounders) that generated the datasets in function
    of the number of the CBN's latent confounders.}
\end{table*}

Table~\ref{dom_dag_tab} reports the results of experiments in which the
datasets were generated from Insurance with two latent confounders. Here,
the domain sizes of these confounders vary from 2 to 4.

\begin{table*}[ht]
  \centering
  \scalebox{1}{
    \begin{tabular}{@{}c@{\quad}cc@{\quad}*{4}{r@{\quad}}c@{\quad}*{4}{r@{\quad}}c@{\quad}*{4}{r@{\quad}}@{}}
      \hline
      \toprule
      & & \multicolumn{5}{c}{\textbf{Algorithm 1}} & \multicolumn{5}{c}{\textbf{MIIC}} & \multicolumn{5}{c}{\textbf{FCI}}\\
      \cmidrule(l){3-7} \cmidrule(l){8-12} \cmidrule(l){13-17}
      $|\dom{L_i}|$&$|\DDD|$&ok&miss&rev.~&type&xs~~~&ok&miss&rev.~&type&xs~~~&ok&miss&rev.~&type&xs~~~\\
      \midrule
      \multirow{5}{*}{2}
      &5000	&{\bf32.38}  &14.44	  &2.46	  &{\bf6.72}  &{\bf3.82}  &29.20  &{\bf12.70}  &3.60  &10.50	&12.44	&23.90	&21.60	&{\bf0.46}  &10.04      &18.06\\
      &10000	&{\bf34.68}  &12.22	  &2.38	  &{\bf6.72}  &{\bf4.06}  &29.86  &{\bf11.70}  &3.52  &10.92	&11.46	&25.76  &18.64	&{\bf0.30}  &11.30      &17.72\\
      &20000	&{\bf38.76}  &{\bf 9.20}  &3.34	  &{\bf4.70}  &{\bf5.04}  &30.42  &10.00       &3.74  &11.84	&13.12	&27.98  &16.18	&{\bf0.14}  &11.70	&17.08\\
      &50000	&{\bf41.08}  &{\bf 7.58}  &4.42	  &{\bf2.92}  &{\bf6.42}  &31.32  & 8.34       &4.00  &12.34	&12.94	&29.42  &13.92	&{\bf0.10}  &12.56	&17.70\\
      &100000	&{\bf40.34}  &{\bf 7.08}  &4.94	  &{\bf3.64}  &{\bf7.48}  &31.46  & 7.88       &3.88  &12.78	&13.40	&31.36  &12.64	&{\bf0.12}  &11.88	&16.96\\
      \midrule
      \multirow{5}{*}{3}
      & 5000	&{\bf34.74}  &13.78	  &1.76	  &{\bf5.72}  &{\bf2.90}  &28.76  &{\bf11.58}  &3.30  &12.36	&15.88	&23.02	&21.40	&{\bf0.64}  &10.94	&21.04\\
      &10000	&{\bf38.00}  &11.86	  &1.96	  &{\bf4.18}  &{\bf3.58}  &29.22  &{\bf11.04}  &3.04  &12.70	&12.42	&25.84	&18.54	&{\bf0.22}  &11.40	&18.90\\
      &20000	&{\bf40.16}  & 9.26	  &3.16	  &{\bf3.42}  &{\bf5.24}  &30.20  &{\bf 9.20}  &3.62  &12.98	&15.58	&27.86	&16.16	&{\bf0.28}  &11.70	&18.36\\
      &50000	&{\bf41.76}  &{\bf 7.72}  &3.92	  &{\bf2.60}  &{\bf6.54}  &30.48  & 8.12       &3.66  &13.74	&16.40	&29.42	&14.24	&{\bf0.22}  &12.12	&18.24\\
      &100000   &{\bf40.84}  &{\bf7.32}   &4.42	  &{\bf3.42}  &{\bf8.04}  &31.68  & 7.42       &3.36  &13.54    &18.24	&31.98	&11.40	&{\bf0.12}  &12.50	&15.52\\
      \midrule
      \multirow{5}{*}{4}
      & 5000	&{\bf34.22}  &13.28	  &1.88	  &{\bf6.62}  &{\bf3.00}  &27.28  &{\bf11.42}  &3.80  &13.50	&16.38	&22.18	&22.06	&{\bf0.68}  &11.08	&22.90\\
      &10000	&{\bf39.72}  &10.86	  &2.04	  &{\bf3.38}  &{\bf3.44}  &30.38  &{\bf 9.70}  &3.64  &12.28	&15.70	&25.20	&18.12	&{\bf0.38}  &12.30	&20.52\\
      &20000	&{\bf39.92}  & 9.06	  &2.94	  &{\bf4.08}  &{\bf5.08}  &29.80  &{\bf 8.44}  &4.00  &13.76	&17.18	&26.92	&16.22	&{\bf0.24}  &12.62	&20.18\\
      &50000	&{\bf42.38}  & 7.86	  &3.88	  &{\bf1.88}  &{\bf7.08}  &30.58  &{\bf 7.78}  &3.70  &13.94	&20.40	&29.52	&13.28	&{\bf0.16}  &13.04	&18.46\\
      &100000   &{\bf42.24}  & 6.74	  &4.26	  &{\bf2.76}  &{\bf7.92}  &30.76  &{\bf 6.52}  &4.64  &14.08	&25.40	&31.16	&11.96	&{\bf0.14}  &12.74	&17.90\\
      \bottomrule
    \end{tabular}}
  \caption{\label{dom_dag_tab}Comparisons of the learnt CPDAGs with those
    of the CBNs (with 2 latent confounders) that generated the datasets, in function
    of the domain sizes of the CBN's latent confounders.}
\end{table*}

\subsection{The impact of the parents number's limit of CPBayes}

In the experiments section of the paper, the first step of Algorithm~1 is
performed by CPBayes \citep{beek-hoff15}. As such, CPBayes takes as input
so-called {\em instances} that are computed from Datasets $\DDD$. 
These contain all the possible nodes' sets that CPBayes will consider as
potential parent sets and, to control the combinatorial explosion, this
requires limiting the number of possible parents of the nodes. 
In the experiments of Section~\ref{expe_sec}, we set this limit to 4. In the literature,
people also fix it to 5 or 6, see, e.g., \cite*{beek-lee17} or \cite*{troes-givr-kats22}.
This makes sense because, in classical Bayesian networks, nodes seldom have
more than 6 parents. In addition, in practical situations, the number of
parents that can be possibly learnt is limited by the size of the dataset
(e.g., to be meaningful, independence tests require contingency tables much
smaller than the dataset size). Tables~\ref{cpbayes_lim1_tab} and \ref{cpbayes_lim2_tab}
show the impact of these limits on the determination of the latent
confounder and on the learnt structure respectively. In these tables
columns ``N:\#'' report the results obtained by limiting CPBayes instances to
have at most \# parents. As can be observed, there is no noticeable
difference increasing the limit from 4 to 6. This is probably due to the
fact that the benchmark Bayesian networks (child, water, insurance) that
generated the datasets have at most 4 parents.

\begin{table*}[ht!]
  \centering
  \scalebox{.95}{
    \begin{tabular}{@{}c@{\ }*{19}{c}@{}}
      \hline
      \toprule
      & & \multicolumn{3}{c}{ok}  &\multicolumn{3}{c}{$\neg$ok}  &\multicolumn{3}{c}{precision}  &\multicolumn{3}{c}{recall}  &\multicolumn{3}{c}{F1} &\multicolumn{3}{c}{time}\\
      \cmidrule(l){3-5} \cmidrule(l){6-8} \cmidrule(l){9-11} \cmidrule(l){12-14} \cmidrule(l){15-17} \cmidrule(l){18-20}
      CBN &$|\DDD|$  & N:4 & N:5 & N:6 & N:4 & N:5 & N:6 & N:4 & N:5 & N:6 & N:4 & N:5 & N:6 & N:4 & N:5 & N:6 & N:4 & N:5 & N:6 \\
      \midrule
      \multirow{5}{*}{child}
      &5000	&0.50	&0.50  &0.50  &0.06  &0.06  &0.06  &0.89  &0.89  &0.89  &0.25  &0.25  &0.25  &0.39  &0.39  &0.39  &0.014  &0.013  &0.016\\
      &10000	&0.76	&0.76  &0.76  &0.18  &0.18  &0.18  &0.81  &0.81  &0.81  &0.38  &0.38  &0.38  &0.52  &0.52  &0.52  &0.028  &0.029  &0.027\\
      &20000	&1.10	&1.10  &1.10  &0.12  &0.12  &0.12  &0.90  &0.90  &0.90  &0.55  &0.55  &0.55  &0.68  &0.68  &0.68  &0.050  &0.068  &0.059\\
      &50000	&1.42	&1.44  &1.44  &0.14  &0.12  &0.12  &0.91  &0.92  &0.92  &0.71  &0.72  &0.72  &0.80  &0.81  &0.81  &0.103  &0.119  &0.130\\
      &100000	&1.44	&1.44  &1.44  &0.14  &0.14  &0.14  &0.91  &0.91  &0.91  &0.72  &0.72  &0.72  &0.80  &0.80  &0.80  &0.190  &0.223  &0.225\\
      \midrule                                                                                                                              
      \multirow{5}{*}{water}                                                                                                                
      &5000	&0.16	&0.16  &0.16  &1.56  &1.56  &1.56  &0.09  &0.09  &0.09  &0.08  &0.08  &0.08  &0.09  &0.09  &0.09  &0.017  &0.022  &0.016\\
      &10000	&0.44	&0.44  &0.44  &0.04  &0.04  &0.04  &0.92  &0.92  &0.92  &0.22  &0.22  &0.22  &0.35  &0.35  &0.35  &0.043  &0.050  &0.050\\
      &20000	&0.80	&0.80  &0.80  &0.22  &0.22  &0.22  &0.78  &0.78  &0.78  &0.40  &0.40  &0.40  &0.53  &0.53  &0.53  &0.104  &0.131  &0.120\\
      &50000	&1.16	&1.16  &1.16  &0.60  &0.60  &0.60  &0.66  &0.66  &0.66  &0.58  &0.58  &0.58  &0.62  &0.62  &0.62  &0.132  &0.157  &0.159\\
      &100000	&1.22	&1.22  &1.22  &1.04  &1.04  &1.04  &0.54  &0.54  &0.54  &0.61  &0.61  &0.61  &0.57  &0.57  &0.57  &0.248  &0.333  &0.318\\
      \midrule                                                                                                                              
      \multirow{5}{*}{\begin{tabular}{@{}c@{}}insu-\\rance\end{tabular}}                                                                    
      &5000	&0.20	&0.20  &0.20  &0.20  &0.20  &0.20  &0.50  &0.50  &0.50  &0.10  &0.10  &0.10  &0.17  &0.17  &0.17  &0.042  &0.044  &0.042\\
      &10000	&0.30	&0.30  &0.30  &0.28  &0.28  &0.28  &0.52  &0.52  &0.52  &0.15  &0.15  &0.15  &0.23  &0.23  &0.23  &0.079  &0.086  &0.090\\
      &20000	&0.62	&0.60  &0.60  &0.30  &0.34  &0.34  &0.67  &0.64  &0.64  &0.31  &0.30  &0.30  &0.42  &0.41  &0.41  &0.172  &0.192  &0.198\\
      &50000	&1.16	&1.14  &1.14  &0.54  &0.42  &0.42  &0.68  &0.73  &0.73  &0.58  &0.57  &0.57  &0.63  &0.64  &0.64  &0.373  &0.429  &0.419\\
      &100000	&1.32	&1.32  &1.32  &0.72  &0.54  &0.54  &0.65  &0.71  &0.71  &0.66  &0.66  &0.66  &0.65  &0.68  &0.68  &0.791  &1.009  &0.994\\
      \bottomrule
  \end{tabular}}
  \caption{\label{cpbayes_lim1_tab}Confounders found for different limits on the number of parents used by CPBayes.}
\end{table*}

\begin{table*}[!ht]
  \centering
  \scalebox{.95}{
    \begin{tabular}{@{}c@{\ }*{16}{c}@{}}
      \hline
      \toprule
      & & \multicolumn{3}{c}{ok} & \multicolumn{3}{c}{miss} & \multicolumn{3}{c}{reversed} & \multicolumn{3}{c}{type} & \multicolumn{3}{c}{excess} \\
      \cmidrule(l){3-5} \cmidrule(l){6-8} \cmidrule(l){9-11} \cmidrule(l){12-14} \cmidrule(l){15-17} 
      CBN & $|\DDD|$ & N:4 & N:5 & N:6 & N:4 & N:5 & N:6 & N:4 & N:5 & N:6 & N:4 & N:5 & N:6 & N:4 & N:5 & N:6 \\
      \midrule
      \multirow{5}{*}{child}
      &5000	&22.46  &22.46  &22.46  & 3.08  & 3.08  & 3.08  &1.02  &1.02  &1.02  &2.44  &2.44  &2.44  &2.54  &2.54  &2.54 \\
      &10000	&23.10  &23.10  &23.10  & 2.52  & 2.52  & 2.52  &1.52  &1.52  &1.52  &1.86  &1.86  &1.86  &3.12  &3.12  &3.12 \\
      &20000	&22.88  &22.88  &22.88  & 1.84  & 1.84  & 1.84  &2.22  &2.22  &2.22  &2.06  &2.06  &2.06  &3.62  &3.64  &3.64 \\
      &50000	&22.98  &22.98  &22.98  & 1.22  & 1.18  & 1.18  &2.84  &2.88  &2.88  &1.96  &1.96  &1.96  &4.34  &4.36  &4.36 \\   
      &100000	&23.20  &23.20  &23.20  & 1.14  & 1.14  & 1.14  &2.88  &2.88  &2.88  &1.78  &1.78  &1.78  &4.50  &4.56  &4.56 \\   
      \midrule                                                                                                              
      \multirow{5}{*}{water}                                                                                                
      &5000	&16.72  &16.72  &16.72  &40.18  &40.18  &40.18  &4.72  &4.72  &4.72  &8.38  &8.38  &8.38  &4.60  &4.60  &4.60 \\
      &10000	&19.00  &19.00  &19.00  &36.98  &36.98  &36.98  &7.58  &7.58  &7.58  &6.44  &6.44  &6.44  &5.96  &5.96  &5.96 \\
      &20000	&21.12	&21.12  &21.12  &36.64  &36.64  &36.64  &6.82  &6.82  &6.82  &5.42  &5.42  &5.42  &8.72  &8.72  &8.72 \\
      &50000	&26.22  &26.22  &26.22  &32.48  &32.48  &32.48  &7.20  &7.20  &7.20  &4.10  &4.10  &4.10  &7.02  &7.02  &7.02 \\
      &100000	&29.94  &29.94  &29.94  &29.42  &29.42  &29.42  &7.10  &7.10  &7.10  &3.54  &3.54  &3.54  &7.06  &7.06  &7.06 \\
      \midrule                                                                                                              
      \multirow{5}{*}{insurance}                                                                                            
      &5000	&32.38  &32.38  &32.38  &14.44  &14.44  &14.44  &2.46  &2.46  &2.46  &6.72  &6.72  &6.72  &3.82  &3.82  &3.82 \\
      &10000	&34.68  &34.68  &34.68  &12.22  &12.22  &12.22  &2.38  &2.38  &2.38  &6.72  &6.72  &6.72  &4.06  &4.06  &4.06 \\
      &20000	&38.76  &38.76  &38.76  & 9.20  & 9.24  & 9.24  &3.34  &3.30  &3.30  &4.70  &4.70  &4.70  &5.04  &5.04  &5.04 \\
      &50000	&41.08  &41.56  &41.56  & 7.58  & 7.60  & 7.60  &4.42  &4.26  &4.26  &2.92  &2.58  &2.58  &6.42  &6.06  &6.06 \\
      &100000	&40.34  &41.06  &41.06  & 7.08  & 7.08  & 7.08  &4.94  &4.76  &4.76  &3.64  &3.10  &3.10  &7.48  &6.90  &6.90 \\
      \bottomrule
    \end{tabular}}
  \caption{\label{cpbayes_lim2_tab}Comparisons of the learnt CPDAGs with those of the CBNs (with confounders) that generated the datasets.}
\end{table*}

\section{Converting a causal model into a Bayesian network}

In this section, we show on an example how to convert a causal model, as
defined in Definition~1 of the paper, into a Bayesian network. For this
purpose, consider the causal model of Figure~\ref{cm_fig1}, where $\XXX =
\{A,B,C,D,E,F\}$. The domain sizes of all the random disturbances $\xi_i$
are equal to $\{1,2,3,4\}$. Those of the variables of $\XXX$ will be
clearly identified from their assigned parameters.

\begin{figure}[htb]
  \centerline{
    \begin{tikzpicture}[
      mynode/.style={circle, draw, font=\scriptsize,
        bottom color = red!5, top color = red!30,
        black, circular drop shadow , text = black, 
        inner sep=0.4mm, minimum size=4.5mm,
        node distance=6mm and 4mm},      
      mynode2/.style={circle, dotted, draw, font=\scriptsize,
        bottom color = BlockBodyColor!50, top color = BlockBodyColor,
        black, circular drop shadow , text = black, 
        inner sep=0.4mm, minimum size=4.5mm,
        node distance=6mm and 4mm},
      myedge/.style={thick, ->,
        >={Stealth[inset=0pt,length=8pt,angle'=28,round]}},
      myedge2/.style={myedge, decorate,decoration={snake,amplitude=1.5pt,pre length=1pt,post length=3.3pt}}
      ]
      \node[mynode] (E) {$E$};
      \node[mynode, above left=  of E] (C) {$C$};
      \node[mynode, above right= of E] (D) {$D$};
      \node[mynode, above left=  of C] (A) {$A$};
      \node[mynode, above right= of C] (B) {$B$};
      \node[mynode, below= of E] (F) {$F$};

      \draw[myedge] (A) -- (C);
      \draw[myedge] (B) -- (C);
      \draw[myedge] (B) -- (D);
      \draw[myedge] (C) -- (E);
      \draw[myedge] (D) -- (E);
      \draw[myedge] (E) -- (F);

      \node[mynode2, above left=  of A] (eA) {$\xi_A$};
      \node[mynode2, above left=  of B] (eB) {$\xi_B$};
      \node[mynode2, left= 7mm of C] (eC) {$\xi_C$};
      \node[mynode2, left= 7mm of E] (eE) {$\xi_E$};
      \node[mynode2, left= 7mm of F] (eF) {$\xi_F$};
      \node[mynode2, above right= of D] (eD) {$\xi_D$};

      \draw[myedge2] (eA) -- (A);
      \draw[myedge2] (eB) -- (B);
      \draw[myedge2] (eC) -- (C);
      \draw[myedge2] (eE) -- (E);
      \draw[myedge2] (eF) -- (F);
      \draw[myedge2] (eD) -- (D);
        
    \end{tikzpicture}
  }
  \caption{\label{cm_fig1}The structure of a causal model.}
\end{figure}
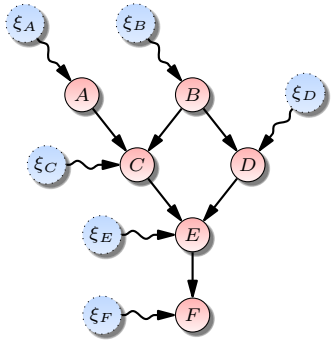

The probability distributions of the random disturbances are defined as
follows:
\begin{center}
  $
  \renewcommand{\extrarowheight}{6pt}
\begin{array}{l}
  P(\xi_A) = \renewcommand{\extrarowheight}{3pt}
  \begin{array}{|*{4}{c|}}
    \hline
    0.2 & 0.1 & 0.3 & 0.4 \\
    \hline
  \end{array}
  \\
  P(\xi_B) = \renewcommand{\extrarowheight}{3pt}
  \begin{array}{|*{4}{c|}}
    \hline
    0.2 & 0.4 & 0.3 & 0.1 \\
    \hline
  \end{array}
  \\
  P(\xi_C) = \renewcommand{\extrarowheight}{3pt}
  \begin{array}{|*{4}{c|}}
    \hline
    0.3 & 0.3 & 0.2 & 0.2 \\
    \hline
  \end{array}
  \\
  P(\xi_D) = \renewcommand{\extrarowheight}{3pt}
  \begin{array}{|*{4}{c|}}
    \hline
    0.5 & 0.3 & 0.1 & 0.1 \\
    \hline
  \end{array}
  \\
  P(\xi_E) = \renewcommand{\extrarowheight}{3pt}
  \begin{array}{|*{4}{c|}}
    \hline
    0.1 & 0.2 & 0.3 & 0.4 \\
    \hline
  \end{array}
  \\
  P(\xi_F) = \renewcommand{\extrarowheight}{3pt}
  \begin{array}{|*{4}{c|}}
    \hline
    0.2 & 0.3 & 0.3 & 0.2 \\
    \hline
  \end{array}
  \\
\end{array}
$
\end{center}

Definition~1 of the paper refers to Definition~2.2.2 from \cite*{pear09}. As
such, the parameters $f_i$ assigned to the variables of $\XXX$ are
deterministic functions (see \cite[p.~68]{pear09}). In our example, the
parameters of the causal model are the following:

\begin{displaymath}
\renewcommand{\extrarowheight}{6pt}
\begin{array}{@{}r@{\ }c@{\ }l@{}}
  f_A(\xi_A) & = &
  \left\{
    \begin{array}{@{\ }l@{\ \ }l}
      a_1 & \mbox{if } \xi_A \leq 2, \\
      a_2 & \mbox{otherwise}      
    \end{array}
  \right.
  \\
  f_B(\xi_B) & = &
  \left\{
    \begin{array}{@{\ }l@{\ \ }l}
      b_1 & \mbox{if } \xi_B \!\mod 2 = 0, \\
      b_2 & \mbox{otherwise}      
    \end{array}
  \right.
  \\
  f_C(A,B,\xi_C) & = &
  \left\{
    \begin{array}{@{\ }l@{\ \ }l}
      A & \mbox{if } \xi_C \leq 2, \\
      B & \mbox{otherwise}      
    \end{array}
  \right.
  \\
  f_D(B,\xi_D) & = &
  \left\{
    \begin{array}{@{\ }l@{\ \ }l}
      d_1 & \mbox{if } B=b_1 \mbox{ and } \xi_D \!\mod 2 = 1, \\
      d_2 & \mbox{if } B=b_2 \mbox{ and } \xi_D = 2, \\
      d_3 & \mbox{otherwise}      
    \end{array}
  \right.
  \\
  f_E(C,D,\xi_E) & = &
  \left\{
    \begin{array}{@{\ }l@{\ \ }l}
      e_1 & \mbox{if } C \in \{a_1,b_1\} \mbox{ and } \xi_E \leq 2, \\
      e_2 & \mbox{if } C \in \{a_2,b_2\} \mbox{ and } D = d_2, \\
      e_3 & \mbox{otherwise}
    \end{array}
  \right.
  \\
  f_F(E,\xi_F) & = &
  \left\{
    \begin{array}{@{\ }l@{\ \ }l}
      f_1 & \mbox{if } \xi_F = 1, \\
      f_2 & \mbox{if } E \in \{e_1,e_3\} \mbox{ and } \xi_F = 2, \\
      f_3 & \mbox{otherwise}      
    \end{array}
  \right.
\end{array}
\end{displaymath}

Deterministic functions like those above can be equivalently represented as
conditional probability tables (CPT) such that: i)~the variables on the
right side of the conditioning bar are those over which the function is defined;
ii)~the variable on the left side of the conditioning bar is the one corresponding
to the codomain of the function; iii)~the values in the CPT are either 1 or
0, depending on whether the value of the variable on the left side of the
conditioning bar corresponds or not to the value of the function given
those of the variables at the right side of the conditioning
bar. Therefore, the above deterministic functions can be represented as:

\begin{minipage}{\linewidth}
\begin{displaymath}
\begin{array}{@{}r@{\ }c@{\ }l@{}}
  P(A | \xi_A) & = &
  \renewcommand{\extrarowheight}{3pt}
  \begin{array}{@{}*{5}{|c}|}
    \hline
    A \backslash \xi_A & 1 & 2 & 3 & 4 \\
    \hline                   
    a_1 & 1 & 1 & 0 & 0 \\
    \hline
    a_2 & 0 & 0 & 1 & 1 \\
    \hline
  \end{array}
  \\
  \rule{0pt}{11mm}P(B | \xi_B) & = & 
  \renewcommand{\extrarowheight}{3pt}
  \begin{array}{@{}*{5}{|c}|}
    \hline
    B \backslash \xi_B & 1 & 2 & 3 & 4 \\
    \hline
    b_1 & 0 & 1 & 0 & 1 \\
    \hline
    b_2 & 1 & 0 & 1 & 0 \\
    \hline
  \end{array}
  \\
  \rule{0pt}{21mm}P(C |A,B,\xi_C) & = &
  \renewcommand{\extrarowheight}{3pt}
  \begin{array}{@{}*{17}{|c}|}
    \cline{2-17}
    \multicolumn{1}{c|}{} & \multicolumn{8}{c|}{a_1} & \multicolumn{8}{c|}{a_2} \\
    \cline{2-17}
    \multicolumn{1}{c|}{} & \multicolumn{4}{c|}{b_1} & \multicolumn{4}{c|}{b_2} & \multicolumn{4}{c|}{b_1} & \multicolumn{4}{c|}{b_2} \\
    \hline
    C \backslash \xi_C & 1 & 2 & 3 & 4 & 1 & 2 & 3 & 4 & 1 & 2 & 3 & 4 & 1 & 2 & 3 & 4 \\
    \hline
    a_1 & 1 & 1 & 0 & 0 & 1 & 1 & 0 & 0 & 0 & 0 & 0 & 0 & 0 & 0 & 0 & 0 \\
    \hline
    a_2 & 0 & 0 & 0 & 0 & 0 & 0 & 0 & 0 & 1 & 1 & 0 & 0 & 1 & 1 & 0 & 0 \\
    \hline
    b_1 & 0 & 0 & 1 & 1 & 0 & 0 & 0 & 0 & 0 & 0 & 1 & 1 & 0 & 0 & 0 & 0 \\
    \hline
    b_2 & 0 & 0 & 0 & 0 & 0 & 0 & 1 & 1 & 0 & 0 & 0 & 0 & 0 & 0 & 1 & 1 \\
    \hline
  \end{array}
  \\
  \rule{0pt}{16mm}P(D | B,\xi_D) & = &
  \renewcommand{\extrarowheight}{3pt}
  \begin{array}{@{}*{9}{|c}|}
    \cline{2-9}
    \multicolumn{1}{c|}{} & \multicolumn{4}{c|}{b_1} & \multicolumn{4}{c|}{b_2} \\
    \hline
    D \backslash \xi_D & 1 & 2 & 3 & 4 & 1 & 2 & 3 & 4  \\
    \hline
    d_1 & 1 & 0 & 1 & 0 & 0 & 0 & 0 & 0 \\
    \hline
    d_2 & 0 & 0 & 0 & 0 & 0 & 1 & 0 & 0 \\
    \hline
    d_3 & 0 & 1 & 0 & 1 & 1 & 0 & 1 & 1 \\
    \hline
  \end{array}
  \\
  P(E|C,D,\xi_E) & = &
  \rule{0pt}{18mm} \renewcommand{\extrarowheight}{3pt}
  \begin{array}{@{}*{25}{|c}|}
    \cline{2-25}
    \multicolumn{1}{c|}{} & \multicolumn{12}{c|}{a_1} & \multicolumn{12}{c|}{a_2} \\
    \cline{2-25}
    \multicolumn{1}{c|}{} & \multicolumn{4}{c|}{d_1} & \multicolumn{4}{c|}{d_2} & \multicolumn{4}{c|}{d_3}
                          & \multicolumn{4}{c|}{d_1} & \multicolumn{4}{c|}{d_2} & \multicolumn{4}{c|}{d_3} \\
    \hline
    E \backslash \xi_E & 1 & 2 & 3 & 4 & 1 & 2 & 3 & 4 & 1 & 2 & 3 & 4 & 1 & 2 & 3 & 4 & 1 & 2 & 3 & 4 & 1 & 2 & 3 & 4 \\
    \hline
    e_1 & 1 & 1 & 0 & 0 & 1 & 1 & 0 & 0 & 1 & 1 & 0 & 0    & 0 & 0 & 0 & 0 & 0 & 0 & 0 & 0 & 0 & 0 & 0 & 0 \\
    \hline
    e_2 & 0 & 0 & 0 & 0 & 0 & 0 & 0 & 0 & 0 & 0 & 0 & 0    & 0 & 0 & 0 & 0 & 1 & 1 & 1 & 1 & 0 & 0 & 0 & 0 \\
    \hline
    e_3 & 0 & 0 & 1 & 1 & 0 & 0 & 1 & 1 & 0 & 0 & 1 & 1    & 1 & 1 & 1 & 1 & 0 & 0 & 0 & 0 & 1 & 1 & 1 & 1 \\
    \hline
  \end{array}
  \\
  \rule{0pt}{18mm} & & 
  \renewcommand{\extrarowheight}{3pt}
  \begin{array}{@{}*{25}{|c}|}
    \cline{2-25}
    \multicolumn{1}{c|}{} & \multicolumn{12}{c|}{b_1} & \multicolumn{12}{c|}{b_2} \\
    \cline{2-25}
    \multicolumn{1}{c|}{} & \multicolumn{4}{c|}{d_1} & \multicolumn{4}{c|}{d_2} & \multicolumn{4}{c|}{d_3}
                          & \multicolumn{4}{c|}{d_1} & \multicolumn{4}{c|}{d_2} & \multicolumn{4}{c|}{d_3} \\
    \hline
    E \backslash \xi_E & 1 & 2 & 3 & 4 & 1 & 2 & 3 & 4 & 1 & 2 & 3 & 4 & 1 & 2 & 3 & 4 & 1 & 2 & 3 & 4 & 1 & 2 & 3 & 4 \\
    \hline
    e_1 & 1 & 1 & 0 & 0 & 1 & 1 & 0 & 0 & 1 & 1 & 0 & 0    & 0 & 0 & 0 & 0 & 0 & 0 & 0 & 0 & 0 & 0 & 0 & 0 \\
    \hline
    e_2 & 0 & 0 & 0 & 0 & 0 & 0 & 0 & 0 & 0 & 0 & 0 & 0    & 0 & 0 & 0 & 0 & 1 & 1 & 1 & 1 & 0 & 0 & 0 & 0 \\
    \hline
    e_3 & 0 & 0 & 1 & 1 & 0 & 0 & 1 & 1 & 0 & 0 & 1 & 1    & 1 & 1 & 1 & 1 & 0 & 0 & 0 & 0 & 1 & 1 & 1 & 1 \\
    \hline
  \end{array}
  \\
  \rule{0pt}{17mm}P(F | E,\xi_F) & = &
  \renewcommand{\extrarowheight}{3pt}
  \begin{array}{@{}*{13}{|c}|}
    \cline{2-13}
    \multicolumn{1}{c|}{} & \multicolumn{4}{c|}{e_1} & \multicolumn{4}{c|}{e_2} & \multicolumn{4}{c|}{e_3} \\
    \hline
    F \backslash \xi_F & 1 & 2 & 3 & 4 & 1 & 2 & 3 & 4 & 1 & 2 & 3 & 4 \\
    \hline
    f_1 & 1 & 0 & 0 & 0 & 1 & 0 & 0 & 0 & 1 & 0 & 0 & 0 \\
    \hline
    f_2 & 0 & 1 & 0 & 0 & 0 & 0 & 0 & 0 & 0 & 1 & 0 & 0 \\
    \hline
    f_3 & 0 & 0 & 1 & 1 & 0 & 1 & 1 & 1 & 0 & 0 & 1 & 1 \\
    \hline
    \end{array}
\end{array}
\end{displaymath}
\end{minipage}

\clearpage

All the above probability distributions, together with the structure of
Figure~\ref{cm_fig2}.a, form the Bayesian network corresponding to the
causal model of Figure~\ref{cm_fig1}. Unfortunately,
this Bayesian network still includes the disturbance nodes whereas the
one that we are looking for is that of Figure~\ref{cm_fig2}.b. Fortunately, it is
easy to remove these disturbance variables: it is sufficient to
marginalize them out from the joint distribution of the Bayesian
network of Figure~\ref{cm_fig2}.a, as shown below.

\begin{figure}[htb]
  \centerline{
    \begin{tikzpicture}[
      mynode/.style={circle, draw, font=\scriptsize,
        bottom color = red!5, top color = red!30,
        black, circular drop shadow , text = black, 
        inner sep=0.4mm, minimum size=4.5mm,
        node distance=6mm and 4mm},      
      mynode2/.style={circle, dotted, draw, font=\scriptsize,
        bottom color = BlockBodyColor!50, top color = BlockBodyColor,
        black, circular drop shadow , text = black, 
        inner sep=0.4mm, minimum size=4.5mm,
        node distance=6mm and 4mm},
      myedge/.style={thick, ->,
        >={Stealth[inset=0pt,length=8pt,angle'=28,round]}},
      myedge2/.style={myedge, decorate,decoration={snake,amplitude=1.5pt,pre length=1pt,post length=3.3pt}}
      ]
      \node[mynode] (E) {$E$};
      \node[mynode, above left=  of E] (C) {$C$};
      \node[mynode, above right= of E] (D) {$D$};
      \node[mynode, above left=  of C] (A) {$A$};
      \node[mynode, above right= of C] (B) {$B$};
      \node[mynode, below= of E] (F) {$F$};

      \draw[myedge] (A) -- (C);
      \draw[myedge] (B) -- (C);
      \draw[myedge] (B) -- (D);
      \draw[myedge] (C) -- (E);
      \draw[myedge] (D) -- (E);
      \draw[myedge] (E) -- (F);

      \node[mynode, above left=  of A] (eA) {$\xi_A$};
      \node[mynode, above left=  of B] (eB) {$\xi_B$};
      \node[mynode, left= 7mm of C] (eC) {$\xi_C$};
      \node[mynode, left= 7mm of E] (eE) {$\xi_E$};
      \node[mynode, left= 7mm of F] (eF) {$\xi_F$};
      \node[mynode, above right= of D] (eD) {$\xi_D$};

      \draw[myedge] (eA) -- (A);
      \draw[myedge] (eB) -- (B);
      \draw[myedge] (eC) -- (C);
      \draw[myedge] (eE) -- (E);
      \draw[myedge] (eF) -- (F);
      \draw[myedge] (eD) -- (D);

      \node[yshift=-16mm, anchor=north] at ($(C)!0.5!(F)$)
      {\begin{tabular}{l@{\ }l}
         a) & Bayesian network with \\
            & disturbance nodes
       \end{tabular}};

      \node[mynode, right=40mm of E] (E2) {$E$};
      \node[mynode, above left=  of E2] (C2) {$C$};
      \node[mynode, above right= of E2] (D2) {$D$};
      \node[mynode, above left=  of C2] (A2) {$A$};
      \node[mynode, above right= of C2] (B2) {$B$};
      \node[mynode, below= of E2] (F2) {$F$};

      \draw[myedge] (A2) -- (C2);
      \draw[myedge] (B2) -- (C2);
      \draw[myedge] (B2) -- (D2);
      \draw[myedge] (C2) -- (E2);
      \draw[myedge] (D2) -- (E2);
      \draw[myedge] (E2) -- (F2);

      \node[yshift=-16mm, anchor=north] at ($(C2)!0.5!(F2)$) {b) the final Bayesian network};
    \end{tikzpicture}
  }
  \caption{\label{cm_fig2}The Bayesian networks corresponding to the causal model.}
\end{figure}
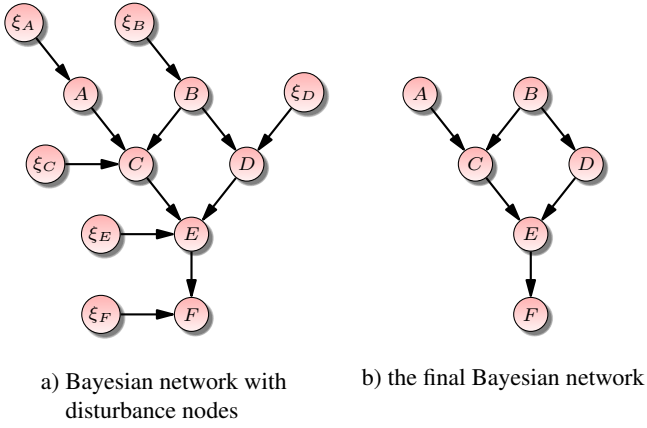

Let $\GGG$ and $\Xi$ denote the graph of Figure~\ref{cm_fig2}.b and Set
$\{\xi_A, \xi_B, \xi_C, \xi_D, \xi_E, \xi_F\}$ respectively.
Then the joint distribution of the Bayesian network of Figure~\ref{cm_fig2}.a is equal to:
\begin{displaymath}
  P(\XXX, \Xi) = \prod_{X \in \XXX} P(\xi_X) \times P(X | \Pa_{\GGG}(X), \xi_X),
\end{displaymath}
and marginalizing out the disturbance variables corresponds to computing:
\begin{eqnarray*}
  P(\XXX) & = & \sum_{\xi \in \Xi} \prod_{X \in \XXX} P(\xi_X) \times
                P(X | \Pa_{\GGG}(X), \xi_X),  \\
          & = & \prod_{X \in \XXX} \sum_{\xi_X} P(\xi_X) \times
                P(X | \Pa_{\GGG}(X), \xi_X) \\
          & = & \prod_{X \in \XXX} P(X | \Pa_{\GGG}(X)), 
\end{eqnarray*}
which corresponds to the decomposition of Figure~\ref{cm_fig2}.b.
Therefore, for each variable $X$ in $\XXX$, we just need to multiply the
CPT of $X$ by the probability distribution of $\xi_X$ and, then,
marginalize out $\xi_X$. This results in the following computations, which
can also be found in the jupyter notebook at
\url{https://pageperso.lis-lab.fr/christophe.gonzales/research/notebooks/ecai2024.ipynb}:

\begin{displaymath}
  \begin{array}{r@{\ }c@{\ }l}
    P(A, \xi_A) & = & P(A | \xi_A) \times P(\xi_A) \\
    & = &
    \renewcommand{\extrarowheight}{3pt}
    \begin{array}{@{}*{5}{|c}|}
      \hline
      A \backslash \xi_A & 1 & 2 & 3 & 4 \\
      \hline                   
      a_1 & 0.2 & 0.1 & 0.0 & 0.0 \\
      \hline
      a_2 & 0.0 & 0.0 & 0.3 & 0.4 \\
      \hline
    \end{array}
  \end{array}
\end{displaymath}

\newpage

Hence, marginalizing out $\xi_A$, we have that:

\begin{displaymath}
  \begin{array}{r@{\ }c@{\ }l}
    P(A) & = & \displaystyle \sum_{\xi_A} P(A, \xi_A) \\
    & = & 
    \renewcommand{\extrarowheight}{3pt}
    \begin{array}{@{}c|c|}
      \cline{2-2}
      a_1 & 0.3 \\
      \cline{2-2}
      a_2 & 0.7 \\
      \cline{2-2}
    \end{array}
  \end{array}
\end{displaymath}

Similarly, we have that:

\begin{displaymath}
  \begin{array}{r@{\ }c@{\ }l}
    P(B) & = & \displaystyle \sum_{\xi_B} P(B | \xi_B) \times P(\xi_B) \\
    & = & \displaystyle \sum_{\xi_B}
    \renewcommand{\extrarowheight}{3pt}
    \begin{array}{@{}*{5}{|c}|}
      \hline
      B \backslash \xi_B & 1 & 2 & 3 & 4 \\
      \hline                   
      b_1 & 0.0 & 0.4 & 0.0 & 0.1 \\
      \hline
      b_2 & 0.2 & 0.0 & 0.3 & 0.0 \\
      \hline
    \end{array} \\
    & = &
    \rule{0pt}{7mm}\renewcommand{\extrarowheight}{3pt}
    \begin{array}{@{}c|c|}
      \cline{2-2}
      b_1 & 0.5 \\
      \cline{2-2}
      b_2 & 0.5 \\
      \cline{2-2}
    \end{array}
  \end{array}
\end{displaymath}

\begin{displaymath}
  \begin{array}{r@{\ }c@{\ }l}
    P(D, \xi_D | B) & = & P(D | B, \xi_D) \times P(\xi_D) \\
    & = &
    \rule{0pt}{16mm}\renewcommand{\extrarowheight}{3pt}
    \begin{array}{@{}*{9}{|c}|}
      \cline{2-9}
      \multicolumn{1}{c|}{} & \multicolumn{4}{c|}{b_1} & \multicolumn{4}{c|}{b_2} \\
      \hline
      D \backslash \xi_D & 1 & 2 & 3 & 4 & 1 & 2 & 3 & 4 \\
      \hline
      d_1 & 0.5 & 0.0 & 0.1 & 0.0 & 0.0 & 0.0 & 0.0 & 0.0 \\
      \hline
      d_2 & 0.0 & 0.0 & 0.0 & 0.0 & 0.0 & 0.3 & 0.0 & 0.0 \\
      \hline
      d_3 & 0.0 & 0.3 & 0.0 & 0.1 & 0.5 & 0.0 & 0.1 & 0.1 \\
      \hline
    \end{array}
  \end{array}
\end{displaymath}

So we have that:

\begin{displaymath}
  \begin{array}{r@{\ }c@{\ }l}
    P(D | B) & = & \displaystyle\sum_{\xi_D} P(D, \xi_D | B) \\
    & = &
    \rule{0pt}{14mm}\renewcommand{\extrarowheight}{3pt}
    \begin{array}{@{}*{3}{|c}|}
      \hline
      D \backslash B & b_1 & b_2 \\
      \hline                   
      d_1 & 0.6 & 0.0 \\
      \hline
      d_2 & 0.0 & 0.3 \\
      \hline
      d_3 & 0.4 & 0.7 \\
      \hline
    \end{array}
  \end{array}
\end{displaymath}

For variables $C$, $E$ and $F$, the tables, which are large,  are provided on the next page.

\newpage

\begin{displaymath}
  \begin{array}{r@{\ }c@{\ }l}
  P(C, \xi_C | A, B) & = & P(C | A, B, \xi_C) \times P(\xi_C) \\
  & = & \rule{0pt}{20mm}
  \renewcommand{\extrarowheight}{3pt}
  \begin{array}{@{}*{17}{|c}|}
    \cline{2-17}
    \multicolumn{1}{c|}{} & \multicolumn{8}{c|}{a_1} & \multicolumn{8}{c|}{a_2} \\
    \cline{2-17}
    \multicolumn{1}{c|}{} & \multicolumn{4}{c|}{b_1} & \multicolumn{4}{c|}{b_2} & \multicolumn{4}{c|}{b_1} & \multicolumn{4}{c|}{b_2} \\
    \hline
    C \backslash \xi_C & 1 & 2 & 3 & 4 & 1 & 2 & 3 & 4 & 1 & 2 & 3 & 4 & 1 & 2 & 3 & 4 \\
    \hline
    a_1 & 0.3 & 0.3 & 0.0 & 0.0 & 0.3 & 0.3 & 0.0 & 0.0 & 0.0 & 0.0 & 0.0 & 0.0 & 0.0 & 0.0 & 0.0 & 0.0 \\
    \hline
    a_2 & 0.0 & 0.0 & 0.0 & 0.0 & 0.0 & 0.0 & 0.0 & 0.0 & 0.3 & 0.3 & 0.0 & 0.0 & 0.3 & 0.3 & 0.0 & 0.0 \\
    \hline
    b_1 & 0.0 & 0.0 & 0.2 & 0.2 & 0.0 & 0.0 & 0.0 & 0.0 & 0.0 & 0.0 & 0.2 & 0.2 & 0.0 & 0.0 & 0.0 & 0.0 \\
    \hline
    b_2 & 0.0 & 0.0 & 0.0 & 0.0 & 0.0 & 0.0 & 0.2 & 0.2 & 0.0 & 0.0 & 0.0 & 0.0 & 0.0 & 0.0 & 0.2 & 0.2 \\
    \hline
  \end{array}
\end{array}
\end{displaymath}

\noindent So we have that:

\begin{displaymath}
  \begin{array}{r@{\ }c@{\ }l}
    P(C|A,B) & = & \displaystyle \sum_{\xi_C} P(C, \xi_C | A, B) \\
    & = &
    \renewcommand{\extrarowheight}{3pt}
    \begin{array}{@{}*{5}{|c}|}
      \cline{2-5}
      \multicolumn{1}{c|}{} & \multicolumn{2}{c|}{a_1} & \multicolumn{2}{c|}{a_2} \\
      \hline
      C \backslash B & b_1 & b_2 & b_1 & b_2 \\
      \hline
      a_1 & 0.6 & 0.6 & 0.0 & 0.0 \\
      \hline
      a_2 & 0.0 & 0.0 & 0.6 & 0.6 \\
      \hline
      b_1 & 0.4 & 0.0 & 0.4 & 0.0 \\
      \hline
      b_2 & 0.0 & 0.4 & 0.0 & 0.4 \\
      \hline
    \end{array}
  \end{array}
\end{displaymath}

\vspace{3mm}

\noindent Similarly, we have that:

\begin{displaymath}
  \begin{array}{r@{\ }c@{\ }l}
    P(E | C, D) & = & \displaystyle \sum_{\xi_E} P(E | C, D, \xi_E) \times P(\xi_E) \\
    & = & \rule{0pt}{15mm}\renewcommand{\extrarowheight}{3pt}
    \begin{array}{@{}*{13}{|c}|}
      \cline{2-13}
      \multicolumn{1}{c|}{} & \multicolumn{3}{c|}{a_1} & \multicolumn{3}{c|}{a_2}  & \multicolumn{3}{c|}{b_1} & \multicolumn{3}{c|}{b_2} \\
      \hline
      E \backslash D & d_1 & d_2 & d_3 & d_1 & d_2 & d_3 & d_1 & d_2 & d_3 & d_1 & d_2 & d_3 \\
      \hline
      e_1 & 0.3 & 0.3 & 0.3 & 0.0 & 0.0 & 0.0 & 0.3 & 0.3 & 0.3 & 0.0 & 0.0 & 0.0 \\
      \hline
      e_2 & 0.0 & 0.0 & 0.0 & 0.0 & 1.0 & 0.0 & 0.0 & 0.0 & 0.0 & 0.0 & 1.0 & 0.0 \\
      \hline
      e_3 & 0.7 & 0.7 & 0.7 & 1.0 & 0.0 & 1.0 & 0.7 & 0.7 & 0.7 & 1.0 & 0.0 & 1.0 \\
      \hline
    \end{array}
  \end{array}
\end{displaymath}

\vspace{3mm}

\noindent Finally, we have that:

\begin{displaymath}
  \begin{array}{r@{\ }c@{\ }l}
     P(F, \xi_F | E) & = & P(F | E, \xi_F) \times P(\xi_F) \\
     & = & \rule{0pt}{15mm}\renewcommand{\extrarowheight}{3pt}
     \begin{array}{@{}*{13}{|c}|}
       \cline{2-13}
       \multicolumn{1}{c|}{} & \multicolumn{4}{c|}{e_1} & \multicolumn{4}{c|}{e_2} & \multicolumn{4}{c|}{e_3} \\
       \hline
       F \backslash \xi_F & 1 & 2 & 3 & 4 & 1 & 2 & 3 & 4& 1 & 2 & 3 & 4 \\
       \hline
       f_1 & 0.2 & 0.0 & 0.0 & 0.0 & 0.2 & 0.0 & 0.0 & 0.0 & 0.2 & 0.0 & 0.0 & 0.0 \\
       \hline
       f_2 & 0.0 & 0.3 & 0.0 & 0.0 & 0.0 & 0.0 & 0.0 & 0.0 & 0.0 & 0.3 & 0.0 & 0.0 \\
       \hline
       f_3 & 0.0 & 0.0 & 0.3 & 0.2 & 0.0 & 0.3 & 0.3 & 0.2 & 0.0 & 0.0 & 0.3 & 0.2 \\
       \hline
     \end{array}
  \end{array}
\end{displaymath}

\noindent So:

\begin{displaymath}
  \begin{array}{r@{\ }c@{\ }l}
    P(F|E) & = & \displaystyle \sum_{\xi_F} P(F, \xi_F | E) \\
    & = &
    \renewcommand{\extrarowheight}{3pt}
    \begin{array}{@{}*{4}{|c}|}
      \hline
      F \backslash E & e_1 & e_2 & e_3\\
      \hline
      f_1 & 0.2 & 0.2 & 0.2 \\
      \hline
      f_2 & 0.3 & 0.0 & 0.3 \\
      \hline
      f_3 & 0.5 & 0.8 & 0.5 \\
      \hline
    \end{array}
  \end{array}
\end{displaymath}

\end{document}